%% file: main.tex
\documentclass{article}
\usepackage[preprint]{neurips_2026}
\usepackage{xcolor}
\usepackage{tikz}
\usetikzlibrary{arrows.meta, positioning, calc, decorations.pathreplacing}
\usepackage{mathtools}
\usepackage{enumitem}

\input{preamble}

\input{macros}

\definecolor{scarlet}{RGB}{255,36,0}

\title{Continuous First, Discrete Later:\\VQ-VAEs Without Dimensional Collapse}

\author{%
  Xinyu Zhao\thanks{Equal contribution.} \\
  MIT \\
  \texttt{lxz@mit.edu} \\
  \And
  Nikita Karagodin\footnotemark[1] \\
  MIT \\
  \texttt{nikitus@mit.edu} \\
  \AND
  Hamed Hassani \\
  University of Pennsylvania \\
  \And
  Sinan Hersek \\
  Google \\
  \And
  Paul Pu Liang \\
  MIT \\
  \And
  Yury Polyanskiy \\
  MIT \\
}

\begin{document}
\maketitle

\input{sections/00_abstract}
\input{sections/01_introduction}

\input{sections/02_background}

\input{sections/03_theory_cleaned}

\input{sections/07_experiments}
\input{sections/08_discussion}

\bibliographystyle{plainnat}
\bibliography{references}

\appendix
\input{sections/09_appendix}

\newpage

\end{document}

%% file: preamble.tex
\usepackage[utf8]{inputenc}
\usepackage[T1]{fontenc}
\usepackage{microtype}

\usepackage{amsmath, amssymb, amsthm, mathtools}
\usepackage{bm}

\usepackage{graphicx}
\usepackage{booktabs}
\usepackage{subcaption}
\usepackage{float}

\usepackage{xcolor}

\usepackage[colorlinks=true, linkcolor=blue, citecolor=blue, urlcolor=blue]{hyperref}

\usepackage[capitalize, noabbrev]{cleveref}

\usepackage{natbib}

\theoremstyle{plain}
\newtheorem{theorem}{Theorem}

\theoremstyle{definition}

\theoremstyle{remark}

\crefname{theorem}{Theorem}{Theorems}
\crefname{proposition}{Proposition}{Propositions}
\crefname{lemma}{Lemma}{Lemmas}
\crefname{corollary}{Corollary}{Corollaries}
\crefname{definition}{Definition}{Definitions}
\crefname{assumption}{Assumption}{Assumptions}
\crefname{remark}{Remark}{Remarks}
\crefname{section}{Section}{Sections}
\crefname{figure}{Figure}{Figures}
\crefname{equation}{Equation}{Equations}

%% file: macros.tex
\providecommand{\nK}{K}

\providecommand{\xvec}{x}
\providecommand{\zvec}{z}
\providecommand{\zqvec}{z_q}
\providecommand{\xhat}{\hat{x}}
\providecommand{\Wone}{W_1}
\providecommand{\Wtwo}{W_2}

\providecommand{\R}{\mathbb{R}}               %

\providecommand{\Twu}{T_{\mathrm{wu}}}       %

\newcommand{\E}{\mathbb{E}}
\newcommand{\calN}{\mathcal{N}}

\DeclareMathOperator{\diag}{diag}
\DeclareMathOperator{\tr}{tr}
\DeclareMathOperator{\sg}{sg}

\newcommand{\rec}{\mathrm{rec}}
\newcommand{\com}{\mathrm{com}}

\usepackage[colorinlistoftodos]{todonotes}

%% file: sections/00_abstract.tex
\begin{abstract}
\vspace{-2mm}
While many approaches to improve VQ-VAE performance focus on codebook size and utilization, the effect of dimensional collapse, where trained VQ-VAE representations live in an extremely low-dimensional subspace ($1-2\%$ of full rank), remains unaddressed. We show theoretically and empirically that dimension collapse causes a hard loss lower bound that various codebook improvement techniques fail to surpass. Our analytic framework extends the sequential learning effect of \cite{saxe2014exact} by introducing ideas from rate-distortion theory and explains how the latent collapse is caused by the VQ suppressing lower-variance directions. Our theory justifies a simple solution: a ``warm-up phase'' that trains the model as an (unquantized) autoencoder before introducing VQ. On both synthetic experiments and large-scale image (VQGAN) and audio (WavTokenizer) VQ-VAEs, we show that AE Warm-Up successfully restores representation dimension, leading to lower reconstruction and perceptual loss at the same training budget. Across codebook sizes $K \in \{2^{10}, 2^{14}, 2^{16}\}$, AE warm-up raises VQGAN codebook effective dimension from 3--5 to 17--19 and reduces rFID by 17--35\%; on WavTokenizer at $K \in \{2^{13}, 2^{14}\}$, it raises codebook dimension from 4 to 17--19 and improves PESQ by 11--14\%.  We empirically characterize how warm-up duration governs the achievable final loss. In agreement with experiment, our theoretical analysis predicts downstream performance as a function of warm-up length, enabling an adaptive criterion for switching from AE Warm-up to VQ-VAE training.
\footnote{Code for our experiments can be found at
\href{https://github.com/xyz-zy/vqvae_latent_span_collapse}{https://github.com/xyz-zy/vqvae\_latent\_span\_collapse}.}

\end{abstract}

%% file: sections/01_introduction.tex
\vspace{-2mm}
\section{Introduction}
\label{sec:intro}
\vspace{-1mm}

Vector-quantized variational autoencoders (VQ-VAEs) are the backbone of modern discrete tokenization for audio and image representation learning and generative models \citep{vandenoord2017vqvae, razavi2019vqvae2, esser2021vqgan, zeghidour2021soundstream, ji2025wavtokenizer}. Scaling the codebook is the standard lever for improving VQ-VAE reconstruction, and a sustained line of work has improved codebook utilization through techniques such as EMA updates, dead-code respawn, and codebook reparameterization \citep{lancucki2020robust, zheng2023cvqvae, zhu2024vqganlc, zhu2025simvq}. It is thus quite surprising that large-scale VQ-VAEs use only a small fraction of allocated latent dimensions: we find that WavTokenizer on LibriTTS, with a 512-dimensional pre-quantization latent, uses only 4 effective dimensions at codebook size $K=2^{12}$ despite 100\% codebook utilization, or as~\cite{zhang2025dcvq} shows, VQGAN on ImageNet collapses to 2--5 effective dimensions across codebook sizes from $K=2^{10}$ to $K=2^{14}$. This dimensional collapse creates an irreducible loss floor that does not improve with codebook scaling. The bottleneck is not the codebook size, but the effective dimension of the latent subspace. 

To theoretically analyze the mechanism behind collapse, we build on the seminal work of~\cite{saxe2014exact,refinetti2022dynamics}, which demonstrated how plain autoencoders learn the principal modes of their data sequentially, in decreasing order of variance. Unfortunately, extending their idea to VQ-VAE is difficult, since it requires analyzing coupled dynamics of the $K$ codebook vectors and encoder-decoder matrices, complicated by the non-differentiable nearest-neighbor quantizer. Our contribution is the introduction of a tractable model, \textit{RD-AE flow}, which replaces VQ with a stochastic noising operation, chosen to minimize mean-squared error subject to mutual information upper bound of $\log_2 K$. We show that under the optimal coupling, latent directions corresponding to low-variance modes may fall below a so-called \textit{water level} and stop evolving due to interaction between the commitment loss term and the STE, see Fig.~\ref{fig:per-mode-dynamics}, thereby collapsing the latent dimension. 

This model shows that the cause of dimensional collapse is the timing of quantization: latent directions that were not yet properly learned when the quantizer turns on are deactivated by the rate constraint and never recover. Our proposed fix trains the encoder and decoder as a plain autoencoder until enough directions are represented by the encoder before introducing VQ; we call this \emph{AE warm-up}. Our theory predicts dependence of the latent effective dimension on the length of the warm-up, and is corroborated by experiments on VQGAN, see Figure~\ref{fig:vqgan_1k_cb16k_3panel}. AE Warm-up raises the VQGAN codebook dimension from 5 to 18 at $K=2^{14}$ and the WavTokenizer codebook dimension from 4 to 21 at $K=2^{16}$, breaking the loss floor that cold-start training hits regardless of codebook size and restoring expected scaling with $K$. Increasing warm-up steps up to the predicted stopping point both increases final codebook dimension and reduces final loss, with gains scaling with target codebook size.
\begin{figure}[t]
    \centering
    \includegraphics[width=1.0\linewidth]{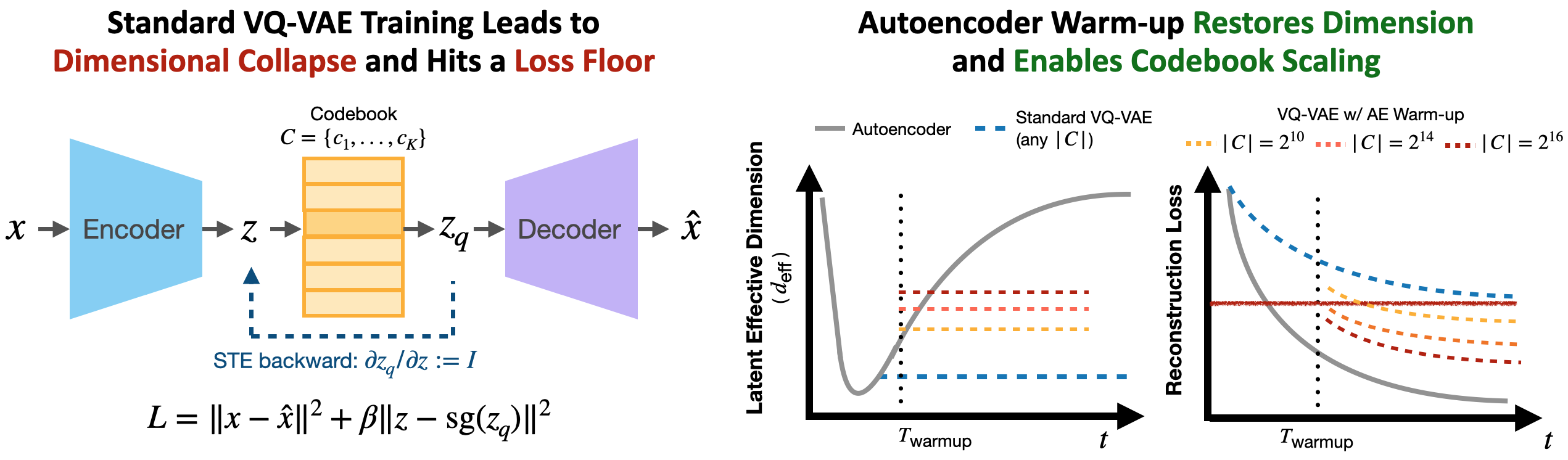}
    \caption{\textbf{Dimensional collapse in VQ-VAEs and the warm-up fix.} During plain autoencoder training (gray), the latent effective dimension $d_{\text{eff}}$ grows as directions of data variance (``modes") are learned sequentially. Turning on the quantizer too early collapses $d_{\text{eff}}$; reconstruction loss hits a floor that is independent of codebook size $|C|$ (blue). Training the encoder--decoder as a plain autoencoder before introducing the quantizer (AE Warm-up) allows the encoder to learn lower-variance directions, preserving higher $d_{\text{eff}}$ after introducing VQ and restoring the expected scaling of reconstruction quality with $|C|$ (yellow-red). We give an information-theoretic account of this mechanism (\cref{sec:rdae}) and validate it in both image (VQGAN trained on ImageNet, \cref{sec:vqgan_validation}) and audio (WavTokenizer trained on LibriTTS, \cref{sec:wavtokenizer_validation}) VQ-VAEs.
    } 
    \label{fig:headline}
\end{figure}

\textbf{Contributions.}
\begin{itemize}[leftmargin=*, topsep=0pt, partopsep=0pt, itemsep=2pt, parsep=0pt]
\item (practice) We demonstrate dimensionality collapse in popular
VQ-VAEs for audio and image
tokenization (WavTokenizer, VQGAN) and show that the resulting performance degradation cannot be fixed
by increasing codebook size $K$ (``loss floor'').
\item (practice) We propose a simple fix, \emph{AE Warm-up}, in which the bottleneck between encoder and decoder remains unquantized during the initial training phase.
\item (theory) We introduce \emph{RD-AE}, a toy model for studying VQ-VAE training dynamics that extends the well-known sequential learning analysis of~\cite{saxe2014exact} to the VQ-VAE setting.
\item (theory+practice) We derive a tight bound on the latent dimensionality and reconstruction loss as a
function of warm-up duration. This allows us to adaptively detect the optimal warm-up duration.
\item (practice) Our approach improves VQGAN (ImageNet-1K) and WavTokenizer (LibriTTS), where AE Warm-up raises codebook effective dimension by 4-5$\times$ and restores loss scaling with $K$.
\end{itemize}

\textbf{Prior work.} 
The warm-up approach has been previously proposed in various ways, although without similar justification or connection to latent dimension. \citet{lancucki2020robust} delay quantization as a stability measure within their codebook-utilization method, \citet{zhao2024representation} report empirical gains from AE pretraining followed by VQ fine-tuning, \citet{wang2025tokenbridge} take this to its limit by applying post-hoc quantization to a frozen pretrained continuous VAE without any VQ training, and (concurrently with this work) \citet{earlyquant2026} propose the warm-up under the name \emph{deferred quantization}, though not connecting it to the latent dimension collapse.

%% file: sections/02_background.tex
\section{Background}
\label{sec:background}

\begin{figure}[t]
  \centering
  \vspace{-0mm}
  \includegraphics[width=0.8\linewidth]{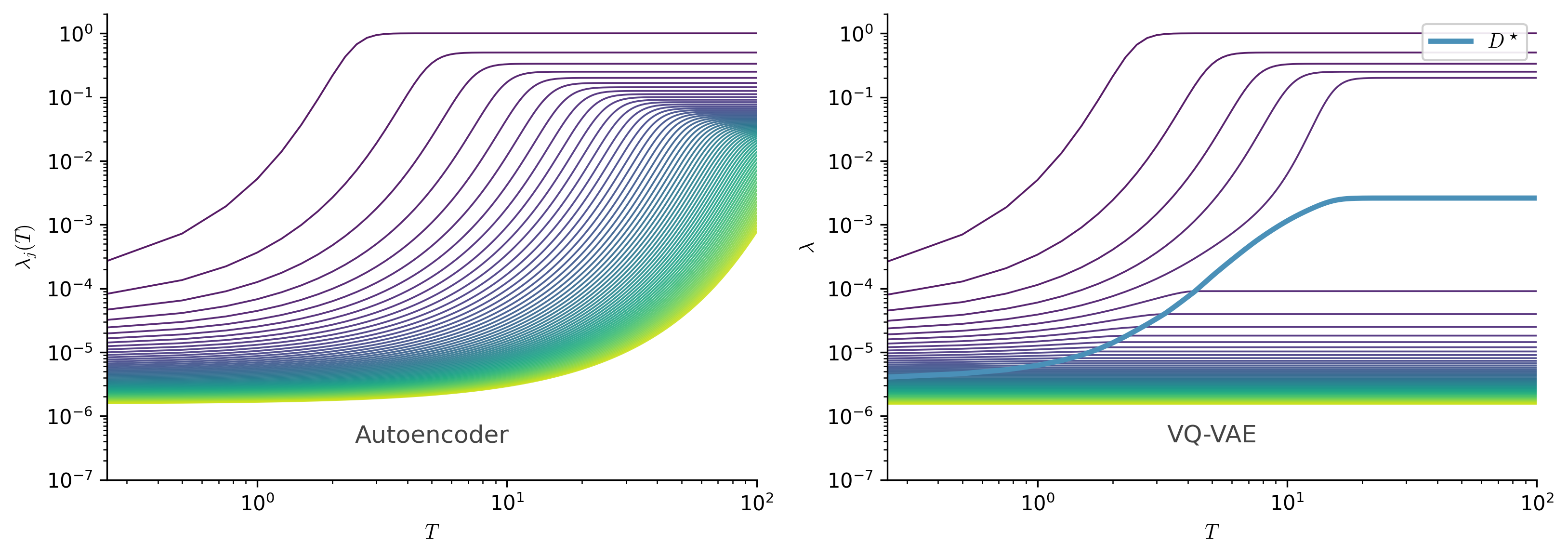}
  \caption{\textbf{Sequential learning in AE and VQ-VAE.} A 2-layer linear AE learns latent modes sequentially~\cite{saxe2014exact} (left). In VQ-VAEs, however, quantizing the bottleneck freezes lower modes (those under the water level $D^*$), thus severely constraining effective dimension (right). See~Section~\ref{sec:rdae} for details.%
}
  \label{fig:per-mode-dynamics}
\end{figure}

\textbf{Autoencoders and vector quantization.} \ \
An autoencoder pairs an encoder $f_\theta : \mathcal{X} \to \mathbb{R}^k$ with a decoder $g_\phi : \mathbb{R}^k \to \mathcal{X}$ and minimizes the reconstruction loss $\|x - g_\phi(f_\theta(x))\|^2$. A vector-quantized autoencoder~\citep{vandenoord2017vqvae} inserts a discretization step between the two: a finite codebook $\mathcal{C} = \{c_1, \dots, c_K\} \subset \mathbb{R}^k$ and a nearest-neighbor quantizer $Q(z) = \arg\min_{c \in \mathcal{C}} \|z - c\|$. The encoder produces a continuous pre-quantization latent $z = f_\theta(x)$; the quantizer replaces it with the nearest codeword $z_q = Q(z)$; the decoder reconstructs $\hat{x} = g_\phi(z_q)$. Training minimizes
\begin{equation}
    \mathcal{L}(\phi,\theta) 
    \;=\; \mathcal{L}_\mathrm{rec}\; +\; \mathcal{L}_\mathrm{com}%
    \;=\; \E[\|x - \hat{x}\|^2 \;+\; \beta \,\|z - \mathrm{sg}(z_q)\|^2],
    \label{eq:vqvae-loss}
\end{equation}
where $\beta$ is the commitment weight and $\mathrm{sg}$ is the stop-gradient and expectation is
over the batch (practice) or population (theory). The encoder receives its gradient through the straight-through estimator~\citep{vandenoord2017vqvae}, $\partial z_q / \partial z := I$. Codebook entries are typically updated by exponential moving averages of assigned latents.

\textbf{Evaluation metrics.}
VQ-VAEs are evaluated on reconstruction quality and on the health of the learned 
codebook. Reconstruction is measured directly via pixel- or sample-level losses 
(L1, L2, mel-spectrogram), and perceptually via task-specific metrics --- LPIPS 
\citep{zhang2018unreasonable} and rFID for images, PESQ \citep{rix2001pesq} and STOI 
\citep{taal2010stoi} for audio. Modern VQ-VAEs additionally include an adversarial discriminator loss \citep{esser2021vqgan, ji2025wavtokenizer}. On the codebook side, \emph{utilization} measures how many allocated codewords are actually in use. 

Even at full utilization, standard VQ-VAE 
representations live in a low-dimensional subspace (typically 4--10 dimensions), as measured by $d_{\mathrm{eff}}$, the number of principal 
components needed to explain $99\%$ of the pre-quantization latent variance:
\begin{equation}
    d_{\mathrm{eff}} = \min\Big\{ m : \textstyle\sum_{i \le m} \lambda_i \big/ 
    \sum_i \lambda_i \ge 0.99 \Big\},
\end{equation}
with $\{\lambda_i\}$ as  the PCA eigenvalues of $z$ in decreasing order. The same 
definition applies to the codebook entries themselves; we report both 
$d_{\mathrm{eff}}(z)$ and $d_{\mathrm{eff}}(\mathcal{C})$ in 
Section~\ref{sec:experiments}.

\textbf{Autoencoder gradient flow.}
\label{sec:ae-gradient-flow}
\citet{saxe2014exact} analyzed the gradient flow of the plain-AE objective
$L_\mathrm{AE}(W_1, W_2) = \E\|x - W_2 W_1 x\|^2$ for the two-layer linear network 
$\hat{X} = W_2 W_1 X$. On the diagonal manifold $W_1 = \diag(u_j)$, $W_2 = \diag(v_j)$, 
the loss decouples as $L_\mathrm{AE}(u, v) = \sum_j \sigma_j^2 (1 - u_j v_j)^2$, and 
the flow acts independently on each mode:
\begin{equation}
  \dot u_j = 2\sigma_j^2\, v_j\,(1 - u_j v_j), \qquad
  \dot v_j = 2\sigma_j^2\, u_j\,(1 - u_j v_j).
  \label{eq:ae-flow}
\end{equation}
With balanced initialization $u_j(0) = v_j(0) = \sqrt{s}$ at small scale $s \in (0,1)$,
the activation for each mode $r_j(t) := u_j(t) v_j(t)$ has the closed form
\begin{equation}
  r_j(t) = \Bigl(1 + \tfrac{1-s}{s}\,e^{-4\sigma_j^2 t}\Bigr)^{-1},
  \label{eq:ae-sigmoid}
\end{equation}
so modes activate sequentially in decreasing order of variance 
(Figure~\ref{fig:per-mode-dynamics}, left). Consequently $d_\text{eff}$ dips as top modes saturate first, then recovers as the rest 
catch up (Figure~\ref{fig:theory-top}, black) — a pattern also visible in large-scale 
VQGAN and WavTokenizer AEs trained with Adam (Figure~\ref{fig:vqgan_1k_cb16k_3panel} and Figure~\ref{fig:ae_activation}).

%% file: sections/03_theory_cleaned.tex
\section{Theoretical model of collapse}
\label{sec:rdae}

We propose to model VQ-VAE training dynamics by replacing the nearest-neighbor quantizer with the
rate-distortion-optimal continuous channel at rate $R = \log_2 K$, resulting in a model that we
call \emph{RD-AE}. This relaxation makes the joint dynamics of the encoder, decoder, and channel
analytically tractable. We study the framework in the classical linear-Gaussian setting where AE
training dynamics admit closed-form analysis~\citep{saxe2014exact}.  Despite the simplifications,
the resulting toy model 
exhibits dimensional collapse (Figure~\ref{fig:theory-top}). Rigorous analysis predicts that
the collapse results from the interplay of two mechanisms: sequential learning in
autoencoders~\citep{saxe2014exact, refinetti2022dynamics} and a fundamental property of MSE-trained
vector-quantizers to collapse low-variance dimensions (which we model here
information-theoretically via rate-distortion bottleneck). 

This section is organized as follows. RD-AE is introduced in Section~\ref{sec:setup},
which also simulates the dense case of the model. The dense case is still difficult to analyze, and thus in Section~\ref{sec:rdae-dynamics} we further simplify it to the case of diagonal covariance matrices and weight initialization, and present the collapse mechanism in detail. Section~\ref{sec:theorem} explains how to predict effective dimension of the resulting VQ-VAE from the state of the AE Warm-up checkpoint used.
Subsequent Section~\ref{sec:experiments} validates that the behavior of non-linear VQ-VAEs trained on real data is in qualitative agreement with the behavior of our toy model.

\subsection{Setup}
\label{sec:setup}

\input{figures/rdae.tex}

\textbf{Architecture.} \ \ The data is $x \sim \calN(0, \Sigma)$ on $\R^d$ with $\Sigma =
\diag(\sigma_1^2, \ldots, \sigma_d^2)$ and $\sigma_1^2 \ge \cdots \ge \sigma_d^2 > 0$. A linear
encoder $W_1 \in \R^{d\times d}$ produces the pre-quantization latent $z = W_1 x$; a (potentially
stochastic) quantizer $Q$ produces $z_q := Q(z)$; a linear decoder $W_2 \in \R^{d\times d}$
returns $\hat x = W_2 z_q$ (Figure~\ref{fig:rd-ae}). Training objective is the standard VQ-VAE
loss~\eqref{eq:vqvae-loss}.

The learning dynamics follow standard VQ-VAE conventions: the encoder reconstruction ``gradient''
uses the straight-through estimator $\partial z_q/\partial z := I$, and the commitment term treats
$z_q$ as constant via stop-gradient. As a result, the update equations for $W_1$ and $W_2$ are not
the gradient of \eqref{eq:vqvae-loss}.

\textbf{Information-theoretic relaxation.} Analyzing VQ's effect on the training dynamics requires jointly tracking $K$ high-dimensional codewords alongside the encoder and decoder. The codebook also breaks the diagonal structure exploited by~\cite{saxe2014exact}, preventing a per-mode decomposition.
Our key insight is that we can replace the dynamically evolving VQ with its
information-theoretic optimal counterpart. Specifically, we assume that the map $z\mapsto z_q$ at
every instant of training is the one that minimizes MSE subject to the information bottleneck constraint:
\begin{equation}\label{eq:rdae-def}
	  \min_{P_{z_q|z}} \left\{\E[\|z-z_q\|^2]: I(z;z_q) \le \log_2 K\right\}\, \triangleq D(R),
\end{equation}
where we denote the minimum MSE as the optimal rate-$R$ distortion in \emph{latent space}, $D(R)$.
This can be understood as relaxing the hard constraint $|\mathrm{supp} z_q| \le K$ to a
softer mutual information constraint. Although we optimistically assume that quantizer is trained
fast enough to remain optimal for the evolving statistics of $z$, RD-AE nevertheless exhibits dimensional collapse. Due to the Gaussianity of $z$, the solution of ~\eqref{eq:rdae-def} is given by a classical waterfilling formula~\cite[Ex.
I.9]{polyanskiy2025information} and conveniently diagonalizes when the initialization and $\Sigma$
does. Thus we obtain a similarly simple and insightful dynamics as in~\cite{saxe2014exact} except
now including the VQ part.

\textbf{RD-AE flow (dense case).} \ \ Training is governed by the system of ODEs on $W_1, W_2$ (derived in
\cref{app:dense-flow}):
\begin{align}
  \dot W_1 &= -2 \big(W_2^\top W_2 \, M \, W_1 - W_2^\top\big) \Sigma - 2 \beta \big(W_1 - M W_1\big) \Sigma, \label{eq:dense-W1} \\
  \dot W_2 &= 2 \Sigma W_1^\top M - 2 W_2 \, U \diag(c_j \lambda_j) U^\top\,, \label{eq:dense-W2}
\end{align}
where $M$ is a covariance of the noise in the  $z\to z_q$ channel, and $c_j$ and $\lambda_j$ are
the spectra of $M$ and $\mathrm{Cov}(z)$, respectively (see \cref{app:dense-flow}). Importantly, when $K \to
\infty$ %
the channel matrix $M \to I$, and the commitment term vanishes. In this limit the RD-AE
flow~\eqref{eq:dense-W1}--\eqref{eq:dense-W2} reduces to the \emph{gradient flow} of the AE
objective in \cref{sec:ae-gradient-flow}, though the $K$ required for this to happen is
impractically large.

\textbf{Numerical simulation (dense case).} \ \
Before deriving the dynamics formally, we preview the qualitative behavior.
Figure~\ref{fig:theory-top} simulates the \textit{RD-AE flow} with $d = 64$, $\sigma_j^2 =
j^{-1}$, $\beta = 1$, with i.i.d.\ Gaussian initialization, compared with the plain-AE limit. The
plain AE drives $\mathcal{L}_\rec \to 0$ while $d_{\textrm{eff}}$ dips and recovers. The RD-AE, by
contrast, exhibits dimensional collapse at every rate $R \in \{10, 14, 18\}$ bits, to
$d_{\textrm{eff}} = 3, 4, 5$ respectively, mirroring the behavior we see in practice. The reconstruction loss plateaus strictly above the Shannon distortion floor $D(R)$.
This gap is the cost of dimensionality collapse.

\begin{figure}[t]
  \centering
  \includegraphics[width=0.8\linewidth]{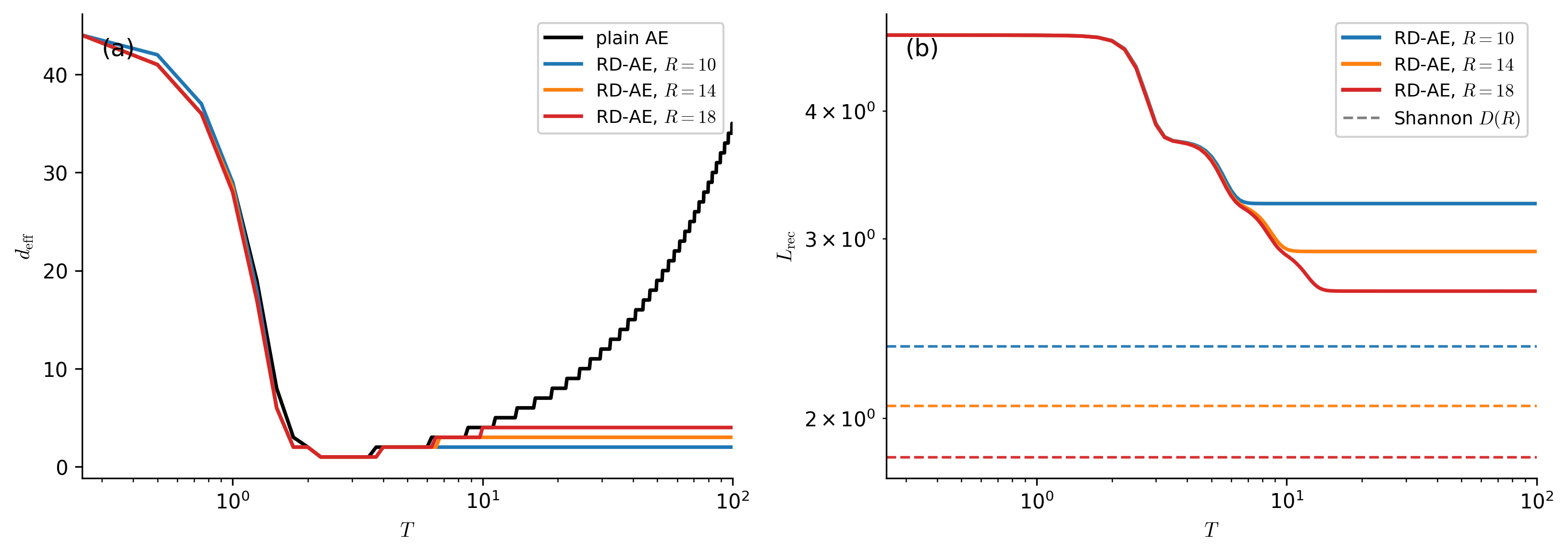}
  \caption{\textbf{Plain AE vs.\ RD-AE flow.} $d = 64$, $\Sigma = \diag (j^{-1})$, $\beta = 1$. $W_1, W_2$ are $d \times d$ matrices with i.i.d.\ entries $(W_i)_{jk} \sim \calN(0, s^2/(4d))$, $s = 0.01$, evolved under the RD-AE flow~\eqref{eq:dense-W1}--\eqref{eq:dense-W2}. Each curve is the median over 128 seeds; cross-seed variance is below $10^{-3}$.
    \textbf{(Left)} Effective dimension $d_\mathrm{eff}$. Plain AE (black) dips and recovers; RD-AE (colored by rate $R$) dips and plateaus.
    \textbf{(Right)} Reconstruction loss. Dashed lines are Shannon floors $D(R)$, achieved by optimal enc-dec pair.}
  \label{fig:theory-top}
\end{figure}

\subsection{Analysis of the RD-AE dynamics in the diagonal case}
\label{sec:rdae-dynamics}

To get more insight into the dynamics, we restrict attention to diagonal $W_1, W_2$: 
$$ W_1(t) =: \diag(u_j(t)), \quad W_2(t) =: \diag(v_j(t)), \quad \Sigma =: \diag(\lambda_j)\,.$$
It is clear that in this case equations~\eqref{eq:dense-W1}-\eqref{eq:dense-W2} decouple and
diagonality at initialization is preserved throughout training. Next, we derive equations for
$u_j,v_j$ and characterize the dimensional collapse mechanism via: (1) a
per-mode flow that reduces to a~\cite{saxe2014exact} logistic with a channel-dependent timescale, (2) a dynamical
water level whose rise pushes active modes below threshold and freezes them, and (3) a closed-form
expression for the resulting loss plateau.
The restriction to the diagonal case is only done for analytic convenience; this simplification still expresses all the major phenomena in the
dense dynamics of
$W_1, W_2$ (cf. \eqref{eq:dense-W1}-\eqref{eq:dense-W2}) that is simulated \cref{fig:per-mode-dynamics} and
\cref{fig:theory-top}.

\textbf{The rate-distortion channel.} \ \
Since $W_1$ is diagonal, $z_j = u_j x_j$ has independent coordinates $z_j \sim \calN(0, \lambda_j)$ with $\lambda_j = u_j^2 \sigma_j^2$. The rate-$R$ channel $P_{z_q \mid z}$ minimizing $\E\|z - z_q\|^2$ under the information constraint factorizes across coordinates~\citep[Sec.~20.4, Ch.~26]{polyanskiy2025information} and is given by reverse water-filling: a unique \emph{water level} $D^\star = D^\star(\lambda_1, \ldots, \lambda_d; R)$, representing the variance threshold below which modes are dropped from the channel, solves
\begin{equation}
  \textstyle\sum_j \tfrac{1}{2} \log_2^+(\lambda_j/D^\star) \;=\; R,
  \label{eq:water-fill}
\end{equation}
and the channel acts coordinatewise as $z_{q,j} \mid z_j \sim \calN(c_j z_j, \tau_j^2)$ with per-mode distortion $D_j$, channel gain $c_j$, and noise variance $\tau_j^2$ given by
\begin{equation}
  D_j \;=\; \min(\lambda_j, D^\star), \qquad c_j \;=\; 1 - D_j/\lambda_j, \qquad \tau_j^2 \;=\; c_j D^\star.
  \label{eq:rd-channel-params}
\end{equation}
Coordinates with $\lambda_j \le D^\star$ satisfy $c_j = \tau_j^2 = 0$ and are passed through as zero; we call these \emph{inactive} and the rest \emph{active}, and write the active set $\mathcal{A}(t) = \{j : \lambda_j(t) > D^\star(t)\}$ with $k(t) = |\mathcal{A}(t)|$.
The channel decomposes as $z_j = z_{q,j} + n_j$ jointly, with $n_j \sim \calN(0, D_j)$ independent of $z_{q,j}$, so that the total latent distortion $\E\|z - z_q\|^2 = \sum_j D_j$ equals $D(R)$ from~\eqref{eq:rdae-def}.

\textbf{Diagonal RD-AE flow.} \ \
Differentiating the full training objective~\eqref{eq:vqvae-loss} requires care. We begin with the reconstruction loss $\mathcal{L}_\rec = \E\|x - W_2 z_q\|^2$, which decomposes per mode using the joint decomposition $z = z_q + n$, coupled across modes through $D^\star$:
\begin{equation}
  \mathcal{L}_\rec(u, v) \;=\; \sum_{j=1}^d \Bigl[\, \sigma_j^2(1 - c_j u_j v_j)^2 \;+\; v_j^2 \tau_j^2 \,\Bigr],
  \label{eq:rdae-rec-diag}
\end{equation}
Under these conditions, the decoder gradient is directly computed from ~\eqref{eq:rdae-rec-diag}. For the encoder, we apply the chain rule with the STE and SG conventions ($\partial z_q / \partial z =: I$ and $\partial \sg(x) / \partial x := 0$) per-sample and take expectation. We get (full derivation in \cref{app:rdae-gradients}): 
\begin{equation}
  \frac{\partial \mathcal{L}_\rec}{\partial v_j} \;=\; -2 c_j u_j \sigma_j^2 (1 - u_j v_j), \qquad \left.\frac{\partial \mathcal{L}_\rec}{\partial u_j}\right|_\mathrm{STE} \;=\; -2 \sigma_j^2 v_j (1 - c_j u_j v_j)
  \label{eq:rdae-grad}
\end{equation}
The encoder commitment gradient gives $\partial \mathcal{L}_\com/\partial u_j = 2 D_j/u_j$. Assembling, we get the diagonal RD-AE flow
\begin{equation}
  \dot u_j \;=\; 2\sigma_j^2 v_j (1 - c_j u_j v_j) \;-\; \frac{2 \beta D_j}{u_j},
  \qquad 
  \dot v_j \;=\; 2\sigma_j^2 c_j u_j (1 - u_j v_j).
  \label{eq:rdae-flow}
\end{equation}
We note that due to the straight-through estimator, %
the flow~\eqref{eq:rdae-flow} is  \textit{not a gradient} flow, since the symmetry condition $\partial \dot u_j / \partial v_j = \partial \dot v_j / \partial u_j$ fails whenever $c_j \neq 1$.

\paragraph{Mode dynamics at balanced initialization.}
We specialize to $\beta = 1$ and balanced initialization $u_j(0) = v_j(0) = \sqrt{s}$. Balance is preserved throughout training: for $j \in \mathcal{A}$, $D_j = D^\star = \lambda_j(1-c_j) = u_j^2 \sigma_j^2 (1-c_j)$, and at $u_j = v_j =: a_j$, we have $\dot u_j - \dot v_j  = 0$.
 
Hence $u_j(t) = v_j(t)$ for all $t$ where mode $j$ is active. Then, $r_j := u_j v_j = u_j^2$ obeys
\begin{equation}
  \dot r_j \;=\; 4\sigma_j^2\, c_j(t)\, r_j(1 - r_j) \quad \text{for } j \in \mathcal{A}, 
  \qquad 
  \dot r_j \;=\; 0 \quad \text{for } j \notin \mathcal{A},
  \label{eq:rdae-logistic}
\end{equation}
derived in Appendix~\ref{app:beta1-reduction}. This is structurally similar to logistic~\eqref{eq:ae-sigmoid} with $c_j(t)$ acting as a time-dependent scale factor: at $c_j(t) \equiv 1$ (the $K \to \infty$ limit), \eqref{eq:rdae-logistic} reduces to \eqref{eq:ae-flow}. For $j \notin \mathcal{A}$, $c_j = 0$ and $D_j = \lambda_j$ give $\dot u_j = \dot v_j = 0$ on the balanced manifold. The choice $\beta = 1$ preserves the encoder-decoder balance, and our experiments support the resulting predictions at $\beta$ values used in practice. %

\textbf{Collapse mechanism.} \ \ Implicit differentiation of~\eqref{eq:water-fill} yields the water level dynamics:
\begin{equation}
  \frac{\dot D^\star}{D^\star} \;=\; \frac{2}{k} \sum_{i \in \mathcal{A}} \frac{\dot u_i}{u_i},
  \label{eq:water-level-dynamics}
\end{equation}
derived in Appendix~\ref{app:water-level-diagonal}.
As active modes grow, $D^\star$ rises through~\eqref{eq:water-level-dynamics}, and modes whose $\lambda_j$ falls below the water line drop out of $\mathcal{A}$ and freeze permanently. The active set is therefore non-increasing,
unlike in the plain AE, where every mode eventually activates. \emph{This is the dimensional collapse phenomenon}: weak modes that have not yet activated when $D^\star$ overtakes them will remain unrepresented in the latent space forever.

\textbf{Loss plateau.} \ \ 
At convergence, with $k_\infty = |\mathcal{A}_\infty|$ active modes surviving, the reconstruction loss is
\begin{equation}
  \mathcal{L}_\rec^\infty \;=\; k_\infty D^\star_\infty \;+\; \sum_{j \notin \mathcal{A}_\infty} \sigma_j^2,
  \label{eq:rdae-loss-infty}
\end{equation}
where the rate budget $R$ is spent only on $\mathcal{A}_\infty$, %
leaving frozen modes unrepresented (paying their full variance $\sigma_j^2$). The plateau therefore exceeds the Shannon optimum $D(R)$ whenever $\mathcal{A}_\infty$ is a strict subset of the optimal active set (Figure~\ref{fig:theory-top}).

\textbf{Special case of $\beta=0$.}  Removing the commitment term by setting $\beta=0$ leads
to very different dynamics: \eqref{eq:rdae-flow} gives $\dot u_j = 2\sigma_j^2 v_j > 0$ on 
inactive modes, so RD-AE always eventually recovers from dimensional collapse at $\beta = 0$. This does not, however, justify dropping the commitment loss in practice: in VQGAN, we find that $\beta=0$ allows dimension recovery when
trained with dead code reset, but achieves worse final loss (Appendix~\ref{app:vqgan_beta0}). We hypothesize that this discrepancy can be attributed to the finite codebook adaptation speed (causing staleness between codewords and statistics of $z$). %

\subsection{Effect of AE Warm-up}
\label{sec:theorem}

Section~\ref{sec:rdae-dynamics} showed that dimensional collapse happens because the rate constraint imposes a rising water level that freezes modes falling below it. A natural fix is \emph{AE Warm-up}: delay quantization until enough modes have activated. This runs the plain-AE flow~\eqref{eq:ae-flow} from balanced diagonal initialization $W_1(0) = W_2(0) = \varepsilon I$ for warm-up duration $T_\mathrm{wu}$, then switch to the RD-AE flow~\eqref{eq:rdae-flow} with rate $R$ and $\beta = 1$. The following theorem bounds how many modes survive the switch as a function of $T_\mathrm{wu}$.

\begin{theorem}
\label{thm:warmup}
Let $x \sim \mathcal{N}(0, \Sigma)$ with $\Sigma = \diag(\sigma_1^2, \ldots, \sigma_d^2)$ and $\sigma_1^2 > \cdots > \sigma_d^2 > 0$. Let $g_j(T) = \big(1 + \tfrac{1-\varepsilon^2}{\varepsilon^2} e^{-4\sigma_j^2 T}\big)^{-1}$ denote the activation of mode $j$ at time $T$ in plain-AE training \eqref{eq:ae-sigmoid}. Define
\begin{equation}
  m^\mathrm{wu}(T_\mathrm{wu}, R)
  \;:=\; \max\!\Bigl\{m \in \{1, \ldots, d\} \;:\; R > \tfrac{1}{2} \sum_{i=1}^{m-1} \log_2 \tfrac{\sigma_i^2\, g_i(T_\mathrm{wu})}{\sigma_m^2\, g_m(T_\mathrm{wu})}\Bigr\}.
  \label{eq:m-wu}
\end{equation}
After switching to RD-AE flow at $T_\mathrm{wu}$, at most $m^\mathrm{wu}$ modes remain active at convergence, and
\begin{equation}
  \mathcal{L}_\rec^\infty
  \;\geq\; m^\mathrm{wu}\, \delta_{m^\mathrm{wu}}(R) \;+\; \sum_{j > m^\mathrm{wu}} \sigma_j^2,
  \qquad
  \delta_m(R) := \Bigl(\prod_{i=1}^{m} \sigma_i^2\Bigr)^{\!1/m} 2^{-2R/m}.
  \label{eq:loss-bound}
\end{equation}
\end{theorem}
\textbf{Proof sketch.} \ \
At the switch time, $\lambda_j(T_\mathrm{wu}) = g_j(T_\mathrm{wu}) \sigma_j^2$, and the condition in~\eqref{eq:m-wu} is the water-filling requirement for the top $m^\mathrm{wu}$ modes to lie above $D^\star(T_\mathrm{wu})$. Modes inactive at $T_\mathrm{wu}$ stay at $r_j = g_j(T_\mathrm{wu})$, contributing variance $\sigma_j^2$ to the reconstruction loss. The bound on $\mathcal{L}_\rec^\infty$ follows from optimal water-filling of rate $R$ over the top $m^\mathrm{wu}$ modes; full proof in Appendix~\ref{app:warmup}.

\textbf{Empirical validation of the warm-up theory.} \ \
Although Theorem~\ref{thm:warmup} is stated for the diagonal RD-AE model (where it is nearly tight; see Appendix~\ref{app:theorem-synthetic}), its core construction transfers directly to non-linear VQ-VAEs. Indeed, one can snapshot a state of AE during the
warm-up stage, compute PCA of the covariance matrix of the latents, evaluate water-filling rate
allocation, and compute the number of surviving dimensions. The resulting number turns out to give
a rather accurate upper bound on the final latent dimension of the trained VQ-VAE if the warm-up is
switched off at the time of the snapshot (see Figure~\ref{fig:vqgan_1k_cb16k_3panel} (right) and
Appendix~\ref{app:vqgan_experiments}). This allows us to make on-the-fly automatic decision about
when to switch from the AE (warm-up) phase and start training the actual VQ-VAE.

%% file: figures/rdae.tex
\begin{figure}[t]
\centering
\begin{tikzpicture}[
    node distance=5mm and 7mm,
    box/.style={draw, rounded corners=2pt, minimum height=7mm, inner sep=4pt},
    every node/.style={font=\small},
    >=stealth
]
    \node (x)    {$\xvec$};
    \node[box, right=of x] (W1) {$\Wone \xvec$};
    \node[right=of W1] (z) {$\zvec$};
    \node[box, right=of z] (Q)
        {$Q(\zvec) = \arg \min_{I(Q(z), z)\leq R} \mathbb{E}\|Q(z) - \zvec\|^2$};
    \node[right=of Q] (zq) {$\zqvec$};
    \node[box, right=of zq] (W2) {$\Wtwo \zqvec$};
    \node[right=of W2] (xh) {$\xhat$};

    \draw[->] (x)  -- (W1);
    \draw[->] (W1) -- (z);
    \draw[->] (z)  -- (Q);
    \draw[->] (Q)  -- (zq);
    \draw[->] (zq) -- (W2);
    \draw[->] (W2) -- (xh);

    \draw[->, thick, dashed, blue!55!black]
        (zq.north) -- ++(0,8mm) -| (z.north);
    \node[font=\footnotesize, blue!55!black, below] 
        at ($(z.north)!0.5!(zq.north) + (0,8mm)$)
        {STE backward: $\partial \zqvec / \partial \zvec := I$};
\end{tikzpicture}
\caption{\textbf{The linear RD-AE.} The linear encoder $\Wone$ maps data to the pre-quantization latent $\zvec$; the (potentially stochastic) quantizer $Q$ assigns $\zvec$ to code $\zqvec$; the linear decoder $\Wtwo$ reconstructs $\xhat$. The backward pass through the quantizer uses the straight-through estimator.}
\label{fig:rd-ae}
\end{figure}

%% file: sections/07_experiments.tex
\begin{figure}[t]
    \centering
    \includegraphics[width=1.0\linewidth]{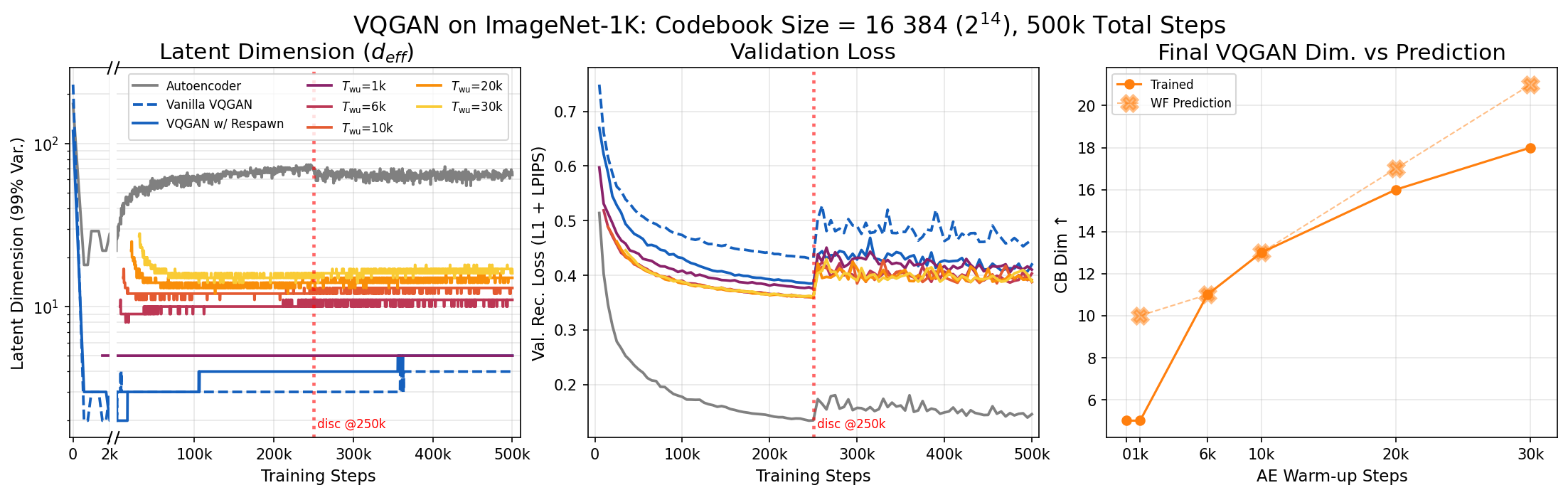}
    \caption{\textbf{VQGAN on ImageNet-1K, $|\mathcal{C}| = 2^{14}$}. Codebook $d_\text{eff}$ (left) and validation reconstruction loss (center) during training for AE warm-up durations $T_\text{wu} \in \{0, 1\text{k}, \dots, 30\text{k}\}$; longer warm-up raises codebook dimension and lowers the loss floor, with returns diminishing past $T_\text{wu} \approx 10\text{k}$. \textit{Right:} reverse water-filling on the AE checkpoint's PCA spectrum at rate $R = 14$ is an empirical upper bound on the final trained codebook $d_\text{eff}$.
    }
    \label{fig:vqgan_1k_cb16k_3panel}
\end{figure}

\section{Escaping Dimensional Collapse with Autoencoder Warm-up}
\label{sec:experiments}

\cref{sec:rdae} argues that VQ-VAE reconstruction is bottlenecked by the latent effective dimension, not the codebook size: modes inactive when the quantizer turns on are frozen by the rate constraint and never recover. The suggested fix, \emph{AE Warm-up}, trains the encoder--decoder as a plain autoencoder for $T_\text{wu}$ steps so that lower-variance directions activate, and then introduces VQ. A single warm-up checkpoint can be reused across codebook sizes, amortizing its cost over a sweep.

\textbf{Choosing $T_\text{wu}$.} \ \
The theory tells us that reverse water-filling at rate $R = \log_2 K$ on the PCA spectrum of the encoder's latents upper-bounds the codebook $d_\text{eff}$ that the downstream VQ-VAE will achieve (\cref{fig:vqgan_1k_cb16k_3panel}, right). In practice, $T_\text{wu}$ should be chosen by monitoring $d_\text{eff}$ during AE training and stopping once it plateaus — further AE training does not lift additional modes above the water level at the target rate. This happens within 20k-40k steps in our ImageNet experiments, and while gains from the additional dimension may plateau earlier than this point based on the dataset, this approach ensures that dimension is not a constraining factor.

\textbf{Empirical validation.} \ \
We test the prediction at scale on both image and audio modalities, training VQGAN on ImageNet (\cref{sec:vqgan_validation}) and WavTokenizer on LibriTTS (\cref{sec:wavtokenizer_validation}). Each setting addresses three questions: \textbf{(Q1)} does cold-start VQ training collapse the latent to a low-dimensional subspace in large-scale architectures; \textbf{(Q2)} does the resulting loss floor persist as $K$ grows; and \textbf{(Q3)} does AE warm-up recover latent dimensionality and restore the expected scaling of reconstruction quality with $K$.

\subsection{VQGAN on ImageNet}
\label{sec:vqgan_validation}
\input{tables/vqgan_100k.tex}

\paragraph{Setup.} We evaluate AE Warm-up by training the f16 VQGAN architecture on ImageNet in two complementary setups. First, we use 20k images from ImageNet-100 to perform controlled ablations over both codebook size and warm-up duration, allowing us to test whether dimensional collapse prevents VQGAN from benefiting from larger codebooks and whether AE Warm-up restores this scaling. We then train on the full ImageNet-1K dataset for 500k steps (adversarial loss enabled at step 250k) at $K=2^{14}$ to test whether the warm-up effect persists at substantially larger scale. 

In both settings, we compare: \textbf{Vanilla VQGAN} (random codebook init); \textbf{VQGAN w/ Respawn}, the standard modern VQ-VAE training method with k-means init, EMA codebook usage (decay 0.99), and dead-code respawn at usage threshold $0.01$~\citep{takida2022sqvae, mentzer2024fsq, dhariwal2020jukebox}; and \textbf{AE Warm-up}, which trains as a plain autoencoder for $T_\text{wu} \in \{1\text{k}, 5\text{k}, 10\text{k}, 15\text{k}, 20\text{k}, 30\text{k}, 40\text{k}\}$ steps before resuming the Respawn protocol. Full training and evaluation details can be found in \cref{app:vqgan_experiment_detail}.

\paragraph{Key findings.} \ \ Table~\ref{tab:vqgan_ablation} shows the final metrics. Without warm-up, neither Vanilla VQGAN nor the Respawn-only baseline improves substantially as $K$ grows. AE Warm-up restores scaling with codebook size and improves all metrics over both baselines: 8.1--18.6\% over Vanilla VQGAN and 5.3--11.0\% over Respawn on L1, 18.0--31.4\% and 14.5--23.1\% on LPIPS, and 30.6--47.5\% and 17.9--35.4\% on rFID. Returns plateau beyond $T_\text{wu}=20\text{k}$.

\textbf{Vanilla VQGAN does not benefit from a larger codebook.} \ \ The three vanilla rows sit at the same L1 loss (0.153) and the same codebook dimension ($d_\text{eff} = 2$) despite $K$ growing by $64\times$, with utilization collapsing from 82.7\% at $K=2^{10}$ to 2.2\% at $K=2^{16}$.

\textbf{Respawn partially rescues the loss floor but cannot recover dead modes.} \ \ VQGAN w/ Respawn improves on Vanilla at every $K$ but falls short of AE warm-up in all cases. Its codebook dimension reaches only 3 at $K \in \{2^{10}, 2^{14}\}$ and 5 at $K = 2^{16}$ — far below warm-up values — consistent with respawn redistributing dead codes within the active subspace rather than reviving collapsed modes.

\textbf{Warm-up lowers the loss floor.} \ \ Increasing $T_\text{wu}$ from 1k to 40k steps monotonically raises codebook effective dimension (from 4--5 to 17--19) and improves LPIPS (from 0.31--0.36 to 0.28--0.33), with noisier but consistent gains on rFID. As predicted by Theorem~\ref{thm:warmup}, each additional warm-up step allows more modes to be represented by the encoder before the quantizer is introduced.

\textbf{AE Warm-up benefits scale to training on ImageNet-1K with adversarial losses.} \ \ Figure~\ref{fig:vqgan_1k_cb16k_3panel} and Table~\ref{tab:vqgan_1k_500k} show the results for $K = 2^{14}$ across several warm-up durations. Cold-start VQGAN, even with respawn, exhibits dimensional collapse, with $d_\text{eff}=4,5$. AE Warm-up progressively raises the final codebook dimension, reaching $d_\text{eff} = 18$ at $T_\text{wu}=30$k; reverse water-filling on the AE checkpoint upper bounds the achievable codebook dimension. The best warm-up duration based on validation loss is $T_\text{wu} = 10$k, and the best rFID ($5.35$) is achieved by $T_\text{wu}=20$k, and further increasing warm-up and dimension plateaus past this point. %

\subsection{WavTokenizer on LibriTTS}
\label{sec:wavtokenizer_validation}
\input{tables/wavtokenizer_100ep.tex}
\paragraph{Setup.} 

To validate our theory in the audio domain and on a larger-scale setting, we train WavTokenizer~\citep{ji2025wavtokenizer} on LibriTTS~\citep{zen2019libritts} with full reconstruction and adversarial losses. We use four codebook sizes ($\nK \in \{4096, 8192, 16384, 65536\}$) and a fixed training budget of 100 epochs in two regimes: (i) cold-start WavTokenizer training (the baseline) and (ii) AE Warm-up for 43 epochs before initializing the codebook with k-means and training as a VQ-VAE with dead code respawn. All other hyperparameters match the original WavTokenizer; details are in \cref{app:wavtokenizer-experiments}.

\paragraph{Key findings.}
Table~\ref{tab:wavtok_comparison} confirms the same pattern as VQGAN. Cold-start training pins codebook $d_\text{eff}$ at $4$ regardless of $K$---fewer than $1\%$ of the 512 available directions---and utilization collapses from $100\%$ at $K{=}2^{12}$ to $8.4\%$ at $K{=}2^{14}$. Performance does not merely plateau but \emph{degrades} with $K$: test-clean PESQ drops from $2.34$ at $K{=}2^{12}$ to $2.18$ at $K{=}2^{13}$. The exception is $K{=}2^{12}$, where a 4-dimensional latent already tiles the codebook and warm-up's extra dimensions ($d_\text{eff}{=}11$) are surplus capacity---consistent with the theory.

AE Warm-up restores monotone scaling: codebook $d_\text{eff}$ rises from $11$ to $21$ as $K$ grows from $2^{12}$ to $2^{16}$, all at near-full utilization. The reconstruction metrics also improve: Mel loss drops from $0.469$ to $0.428$ on test-clean, and PESQ improves from $2.275$ to $2.626$. We did not attempt cold-start at $K{=}2^{16}$ given the severity of collapse already at $K{=}2^{14}$.

%% file: tables/vqgan_100k.tex
\begin{table}[!t]
\centering
\scriptsize
\setlength{\tabcolsep}{3pt}
\caption{\textbf{VQGAN on ImageNet-100 (20k) warm-up ablation.} L1, LPIPS, and rFID: lower is better. D: codebook $d_\text{eff}$ (higher is better). Bold indicates the best value within each codebook size. Vanilla VQGAN uses uniform codebook init and no respawn; VQGAN w/ Respawn adds k-means init and respawn; $T_\text{wu} = $ \{1k, \ldots, 40k\} additionally pre-trains the encoder--decoder as a plain autoencoder for the indicated number of steps before introducing the codebook. All runs share the same 100k-step budget, so warm-up trades VQ steps for AE steps. %
}
\vspace{1em}
\begin{tabular}{l *{15}{r}}
\toprule
& \multicolumn{5}{c}{$|\mathcal{C}| = 2^{10}$}
& \multicolumn{5}{c}{$|\mathcal{C}| = 2^{14}$}
& \multicolumn{5}{c}{$|\mathcal{C}| = 2^{16}$} \\
\cmidrule(lr){2-6} \cmidrule(lr){7-11} \cmidrule(lr){12-16}
Variant
& Util\% & $D$$\uparrow$ & L1$\downarrow$ & LPIPS$\downarrow$ & rFID$\downarrow$
& Util\% & $D$$\uparrow$ & L1$\downarrow$ & LPIPS$\downarrow$ & rFID$\downarrow$
& Util\% & $D$$\uparrow$ & L1$\downarrow$ & LPIPS$\downarrow$ & rFID$\downarrow$ \\
\midrule
Vanilla VQGAN
& 82.7  & 2 & 0.153 & 0.405 & 87.8
& 8.0   & 2 & 0.153 & 0.408 & 81.0
& 2.2   & 2 & 0.153 & 0.406 & 83.6 \\
\qquad w/ Respawn
& 100.0 & 3 & 0.148 & 0.389 & 74.2
& 99.9  & 3 & 0.140 & 0.372 & 69.7
& 47.0  & 5 & 0.140 & 0.363 & 67.9 \\
$T_\text{wu} = $ 1k
& 100.0 & 5 & 0.142 & 0.355 & 69.1
& 100.0 & 4 & 0.134 & 0.338 & 61.9
& 98.4  & 5 & 0.128 & 0.313 & 52.6 \\
$T_\text{wu} = $ 5k
& 100.0 & 7 & 0.141 & 0.343 & 61.7
& 100.0 & 7 & 0.129 & 0.308 & 51.4
& 98.5  & 7 & \textbf{0.124} & 0.294 & 47.4 \\
$T_\text{wu} = $ 10k
& 100.0 & 7 & \textbf{0.141} & 0.339 & 63.4
& 99.9  & 10 & \textbf{0.128} & 0.302 & 52.0
& 98.0  & 12 & \textbf{0.124} & 0.284 & 47.8 \\
$T_\text{wu} = $ 15k
& 100.0 & 10 & 0.142 & 0.337 & 67.0
& 99.9  & 12 & 0.129 & 0.299 & 51.4
& 97.7  & 13 & \textbf{0.124} & 0.282 & 47.3 \\
$T_\text{wu} = $ 20k
& 100.0 & 11 & 0.142 & 0.334 & \textbf{60.9}
& 99.8  & 13 & 0.129 & 0.297 & 52.6
& 97.6  & 14 & 0.125 & 0.281 & 47.4 \\
$T_\text{wu} = $ 30k
& 100.0 & 14 & 0.142 & \textbf{0.332} & 61.8
& 99.8  & 15 & 0.130 & 0.295 & \textbf{49.8}
& 97.2  & 16 & 0.126 & 0.279 & 46.7 \\
$T_\text{wu} = $ 40k
& 100.0 & \textbf{19} & 0.143 & 0.334 & 61.3
& 99.8  & \textbf{17} & 0.131 & \textbf{0.294} & 50.6
& 97.0  & \textbf{18} & 0.126 & \textbf{0.279} & \textbf{43.9} \\
\bottomrule
\end{tabular}
\label{tab:vqgan_ablation}
\end{table}

%% file: tables/wavtokenizer_100ep.tex
\begin{table}[t]
\centering
\setlength{\tabcolsep}{3pt}
\caption{\textbf{WavTokenizer audio reconstruction performance on LibriTTS test-clean/test-other.} Each row pair compares cold-start WavTokenizer training against AE Warm-up at the same codebook size. \textbf{Bold} indicates the winner within each codebook size; \underline{underline} indicates the overall best. Without warm-up, codebook utilization collapses as $K$ grows and codebook effective dimension is pinned at $4$. AE Warm-up maintains near-full utilization at every $K$ and raises codebook effective dimension to $11$--$21$, and causes $D$ (codebook $d_\text{eff}$), Mel, PESQ, and STOI to improve monotonically with $K$. Cold-start at $K=2^{16}$ is omitted given the severity of collapse at $K=2^{14}$.
}
\small
\vspace{1em}
\begin{tabular}{lcccccccc}
\toprule
Method & $|\mathcal{C}|$ & Util$\uparrow$ & $D\uparrow$ & Mel$\downarrow$ & UTMOS$\uparrow$ & PESQ$\uparrow$ & STOI$\uparrow$ & V/UV F1$\uparrow$ \\
\midrule
WavTokenizer  & 4096   & 100.0\% & 4  & 0.480/0.548 & \textbf{\underline{4.051}}/\textbf{\underline{3.563}} & \textbf{2.336}/\textbf{2.054} & \textbf{0.912}/\textbf{0.878} & \textbf{0.940}/\textbf{0.916} \\
AE Warm-up    & 4096   & 100.0\% & 11 & \textbf{0.469}/\textbf{0.526} & 3.980/3.484 & 2.275/2.032 & 0.907/0.873 & 0.931/0.905 \\
\midrule
WavTokenizer  & 8192   & 16.9\%  & 4  & 0.510/0.584 & \textbf{4.009}/\textbf{3.547} & 2.179/1.937 & 0.902/0.867 & \textbf{0.934}/\textbf{0.908} \\
AE Warm-up    & 8192   & 100.0\% & 17 & \textbf{0.448}/\textbf{0.504} & 3.991/3.480 & \textbf{2.425}/\textbf{2.147} & \textbf{0.915}/\textbf{0.882} & 0.933/0.907 \\
\midrule
WavTokenizer  & 16384  & 8.4\%   & 4  & 0.497/0.566 & 4.007/\textbf{3.539} & 2.209/1.977 & 0.904/0.870 & 0.936/0.912 \\
AE Warm-up    & 16384  & 100.0\% & 19 & \textbf{0.432}/\textbf{0.486} & \textbf{4.041}/3.536 & \textbf{2.514}/\textbf{2.229} & \textbf{0.920}/\textbf{0.889} & \textbf{\underline{0.942}}/\textbf{\underline{0.917}} \\
\midrule
AE Warm-up    & 65536  & 99.7\%  & 21 & \underline{0.428}/\underline{0.480} & 4.050/3.521 & \underline{2.626}/\underline{2.328} & \underline{0.925}/\underline{0.895} & 0.940/0.915 \\
\bottomrule
\end{tabular}
\label{tab:wavtok_comparison}
\end{table}

%% file: sections/08_discussion.tex
\section{Discussion and Limitations}
\label{sec:discussion}

The bottleneck in VQ-VAE training is the span of the latent subspace the encoder delivers to the quantizer, not the number of codes. AE warm-up addresses this directly, and because it acts before quantization begins, it is agnostic to the downstream quantizer: residual VQ~\citep{zeghidour2021soundstream, kumar2023dac, zhang2023speechtokenizer, defossez2024moshi}, product quantization~\citep{zhang2025dcvq}, and structural reparameterizations such as SimVQ~\citep{zhu2025simvq} are all compatible. Indeed, the mode-suppression mechanism operates within each residual stage or product factor, so these methods stand to benefit from warm-up as well.

\textbf{Scope and future work.} \ \ Our theory assumes linear encoders and decoders with independent Gaussian inputs --- enough to expose the mechanism, but restrictive. Extending to nonlinear encoders and predicting the loss floor quantitatively in realistic models are natural next steps. The method also raises practical questions: how far can codebook size be pushed once warm-up is in place, and which existing codebook-geometry techniques continue to work in the higher-$d_\text{eff}$ regime warm-up unlocks?

%% file: sections/09_appendix.tex
\section{Proofs for Section~\ref{sec:rdae}}
\label{app:rdae-proofs}

\subsection{Dense RD-AE flow}
\label{app:dense-flow}

We derive \eqref{eq:dense-W1}--\eqref{eq:dense-W2} using the notation of Section~\ref{sec:setup}. All channel parameters $(c_j, \tau_j^2, D_j, D^\star)$ are treated as constant throughout, consistent with the assumption that the quantizer remains optimal for the current latent statistics.

\paragraph{Channel in the eigenbasis.}
We denote the covariance matrix of the encoder latent space as $\Sigma_z$, and write the diagonal decomposition $\Sigma_z = W_1 \Sigma W_1^\top = U \Lambda U^\top$. In coordinates $\tilde{z} = U^\top z$, the components $\tilde{z}_j \sim \calN(0, \lambda_j)$ are independent, so the rate-$R$ channel of Section~\ref{sec:setup} applies coordinatewise with parameters $(c_j, \tau_j^2, D_j)$ as in~\eqref{eq:rd-channel-params}. 

Returning to the original basis, $\E[z_q \mid z] = Mz$ with $M := U\diag(c_1, \ldots, c_d)U^\top$. Since $z_{q,j}$ is the MMSE estimate of $\tilde{z}_j$, the residual $n := z - z_q$ satisfies $z_q \perp n$ by Gaussianity, with $n \sim \calN(0, U\diag(D_1, \ldots, D_d)U^\top)$.

\paragraph{Joint moments.}
The decomposition $z = z_q + n$ with $z_q \perp n$ gives
\begin{equation}
  \E[z_q z_q^\top] = U\diag(c_j \lambda_j) U^\top =: \Gamma_q, \qquad
  \E[z z_q^\top] = \Gamma_q, \qquad
  \E[x z_q^\top] = \Sigma W_1^\top M,
\end{equation}
where $M = U\diag(c_1, \ldots, c_d)U^\top$. The first two follow because $z_q \perp n$ implies $\E[z z_q^\top] = \E[z_q z_q^\top]$. The third uses $z_q = M z + (\text{noise})$ and $z = W_1 x$, with $M$ symmetric.

\paragraph{Decoder gradient.}
From $\mathcal{L}_\rec = \E\|x - W_2 z_q\|^2 = \tr\Sigma - 2\tr(W_2 \E[z_q x^\top]) + \tr(W_2^\top W_2 \Gamma_q)$, we treat the channel parameters $(M, \Gamma_q)$ as constant and differentiate:
\begin{equation}
  \nabla_{W_2} \mathcal{L}_\rec = -2\Sigma W_1^\top M + 2 W_2 \Gamma_q.
  \label{eq:dense-decoder-grad}
\end{equation}

$\mathcal{L}_\com$ does not depend on $W_2$, so $\nabla_{W_2}\mathcal{L}_\com = 0$.

\paragraph{Encoder reconstruction gradient under STE.}
The straight-through estimator uses $\partial z_q / \partial z := I$ in differentiation. %
Equivalently, replace $z_q$ by the surrogate $z_q^\mathrm{surr} = z + \mathrm{sg}(z_q - z) = W_1 x + s$ where $s := z_q - W_1 x$ is treated as a constant. The reconstruction loss becomes
\begin{equation}
  \mathcal{L}_\rec^\mathrm{STE} \;=\; \E\|x - W_2 W_1 x - W_2 s\|^2,
\end{equation}
identical in value to $\mathcal{L}_\rec$ but with $W_1$-dependence routed only through $W_1 x$. Differentiating in $W_1$ with $s$ constant,
\begin{equation}
  \nabla_{W_1} \mathcal{L}_\rec |^\mathrm{STE} \;=\; -2 W_2^\top \E[(x - W_2 z_q) x^\top] \;=\; -2 W_2^\top \Sigma + 2 W_2^\top W_2 \E[z_q x^\top],
\end{equation}
where the residual $x - W_2 W_1 x - W_2 s = x - W_2 z_q$ uses the true $z_q$. Substituting $\E[z_q x^\top] = M W_1 \Sigma$,
\begin{equation}
  \nabla_{W_1} \mathcal{L}_\rec |^\mathrm{STE} \;=\; 2(W_2^\top W_2 M W_1 - W_2^\top) \Sigma.
  \label{eq:dense-encoder-rec-grad}
\end{equation}

\paragraph{Encoder commitment gradient.}
The commitment term $\mathcal{L}_\com = \E\|z - \mathrm{sg}(z_q)\|^2$ treats $z_q$ as constant during differentiation, so it only depends on $W_1$ through $z = W_1 x$:
\begin{equation}
  \mathcal{L}_\com \;=\; \tr(W_1 \Sigma W_1^\top) - 2 \E[z^\top \mathrm{sg}(z_q)] + \mathrm{const}.
\end{equation}
Differentiating, $\nabla_{W_1} \tr(W_1 \Sigma W_1^\top) = 2 W_1 \Sigma$ and $\nabla_{W_1} \E[z^\top \mathrm{sg}(z_q)] = \E[\mathrm{sg}(z_q) x^\top]$ has population mean $M W_1 \Sigma$, giving
\begin{equation}
  \nabla_{W_1} \mathcal{L}_\com \;=\; 2 (W_1 - M W_1) \Sigma.
  \label{eq:dense-encoder-com-grad}
\end{equation}

\paragraph{Assembling.}
Combining \eqref{eq:dense-decoder-grad}, \eqref{eq:dense-encoder-rec-grad}, \eqref{eq:dense-encoder-com-grad}, we arrive at the flow in \eqref{eq:dense-W1}--\eqref{eq:dense-W2}:
\begin{align}
  \dot W_1 \;&=\; -2(W_2^\top W_2 M W_1 - W_2^\top)\Sigma \;-\; 2\beta(W_1 - M W_1)\Sigma, \label{eq:dense-flow-W1}\\
  \dot W_2 \;&=\; 2 \Sigma W_1^\top M - 2 W_2 \Gamma_q. \label{eq:dense-flow-W2}
\end{align}

\subsubsection{Dense water-level evolution} In the main text, we examined the water-level dynamics under diagonal initialization of $W_1$ and $W_2$ (Eq. \eqref{eq:water-level-dynamics}, derivation in \S\ref{app:water-level-diagonal}), which characterizes how $D*$ rises during RD-AE training to remove modes from the active set . We derive the water-level dynamics under dense initialization here.

Recall that the water-fill constraint $\sum_{j \in A} \tfrac{1}{2}\log_2(\lambda_j/D^\star) = R$ couples $D^\star$ to the eigenvalues of $\Sigma_z(t)$, which themselves evolve through $W_1(t)$. To track $\dot \lambda_j$, differentiate the eigenvalue identity $\Sigma_z u_j = \lambda_j u_j$ and the normalization $u_j^\top u_j = 1$:
\begin{equation}
  \dot \Sigma_z \, u_j + \Sigma_z \dot u_j \;=\; \dot \lambda_j u_j + \lambda_j \dot u_j, \qquad u_j^\top \dot u_j = 0.
\end{equation}
Left-multiplying by $u_j^\top$ and using $u_j^\top \Sigma_z = \lambda_j u_j^\top$ gives
\begin{equation}
  \dot \lambda_j \;=\; u_j^\top \dot \Sigma_z\, u_j,
  \qquad \dot \Sigma_z \;=\; \dot W_1 \Sigma W_1^\top + W_1 \Sigma \dot W_1^\top.
  \label{eq:eigval-evolution}
\end{equation}
The water-level evolution follows by differentiating the rate constraint at fixed active set:
\begin{equation}
  \frac{\dot D^\star}{D^\star} \;=\; \frac{1}{k} \sum_{j \in A} \frac{\dot \lambda_j}{\lambda_j}
    \;=\; \frac{1}{k} \sum_{j \in A} \frac{u_j^\top \dot \Sigma_z u_j}{\lambda_j}.
  \label{eq:dense-water-level}
\end{equation}
Between water-crossings, $A$ and $k$ are constant, and \eqref{eq:dense-flow-W1}--\eqref{eq:dense-water-level} close the dynamics of $(W_1, W_2, D^\star)$. At a crossing, $A$ updates discretely, $k$ decrements by one, and $D^\star$ is continuous.

\subsection{RD-AE flow derivation (diagonal case)}
\label{app:rdae-gradients}

We derive the diagonal RD-AE flow~\eqref{eq:rdae-flow} stated in Section~\ref{sec:rdae-dynamics}. Throughout we assume that $W_1 = \diag(u_j)$, $W_2 = \diag(v_j)$. The RD channel of rate $R$ bits acts coordinatewise on $z = W_1 x$.

The channel parameters $c_j, \tau_j^2, D^\star$ are treated as constant while differentiating the loss, as in~\S\ref{app:dense-flow}. However, $c_j$ and $D^\star$ change over time as the latent variances $\lambda_j = u_j^2 \sigma_j^2$ evolve, and this coupling is derived separately in~\S\ref{app:water-level-diagonal}.

\paragraph{Key moments.}
In the diagonal case, the dense decomposition of~\S\ref{app:dense-flow} reduces to $z_j = z_{q,j} + n_j$ with $n_j \sim \calN(0, D_j)$ independent of $z_{q,j}$, giving
\begin{equation}
  \E[z_{q,j}^2] = c_j \lambda_j, \qquad
  \E[z_j \, z_{q,j}] = c_j \lambda_j, \qquad
  \E[x_j \, z_{q,j}] = c_j u_j \sigma_j^2.
\end{equation}

\paragraph{Decoder gradient.}
The decoder direct gradient updates. From $\hat x_j = v_j z_{q,j}$,
\begin{equation}
  \frac{\partial \mathcal{L}_\rec}{\partial v_j}
  \;=\; -2 \, \E\bigl[(x_j - v_j z_{q,j}) z_{q,j}\bigr]
  \;=\; -2 \bigl(\E[x_j z_{q,j}] - v_j \E[z_{q,j}^2]\bigr)
  \;=\; -2 c_j u_j \sigma_j^2 (1 - u_j v_j),
\end{equation}
using $c_j \lambda_j = c_j u_j^2 \sigma_j^2$ in the last step. The commitment term does not depend on $v_j$, so $\partial \mathcal{L}_\com/\partial v_j = 0$.

\paragraph{Encoder reconstruction gradient under STE.}
The straight-through estimator redefines the backward pass to treat $\partial z_{q,j} / \partial z_j := 1$, while the forward pass uses the actual $z_{q,j}$ from the channel. Concretely, STE replaces $z_{q,j}$ with a surrogate $z_{q,j}^\mathrm{surr} = z_j + \mathrm{sg}(z_{q,j} - z_j)$, which equals $z_{q,j}$ in value but has identity Jacobian in $z_j$. Under this surrogate, the reconstruction $\hat x_j = v_j z_{q,j}^\mathrm{surr}$ satisfies $\partial \hat x_j / \partial u_j = v_j \, \partial z_j / \partial u_j = v_j x_j$, and
\begin{equation}
  \frac{\partial \mathcal{L}_\rec}{\partial u_j}\bigg|_\mathrm{STE}
  \;=\; -2 v_j \, \E\bigl[(x_j - v_j z_{q,j}) x_j\bigr]
  \;=\; -2 v_j \bigl(\sigma_j^2 - v_j \E[x_j z_{q,j}]\bigr)
  \;=\; -2 v_j \sigma_j^2 (1 - c_j u_j v_j).
\end{equation}
The value of $z_{q,j}$ inside the residual uses the true channel output, not the surrogate --- only the backward Jacobian is replaced.

\paragraph{Encoder commitment gradient.}
The commitment term $\mathcal{L}_\com = \E \|z - \mathrm{sg}(z_q)\|^2$ treats $z_q$ as constant during differentiation. Since $z_j = u_j x_j$, differentiating with respect to $u_j$ gives:
\begin{equation}
  \frac{\partial \mathcal{L}_\com}{\partial u_j}
  \;=\; 2 \, \E\bigl[(z_j - z_{q,j}) x_j\bigr]
  \;=\; 2 \bigl(u_j \sigma_j^2 - \E[x_j z_{q,j}]\bigr)
  \;=\; 2 u_j \sigma_j^2 (1 - c_j) \;=\; \frac{2 D_j}{u_j},
\end{equation}
where the last equality uses $D_j = \lambda_j(1 - c_j) = u_j^2 \sigma_j^2 (1 - c_j)$.

\paragraph{Assembling the flow.}
The gradient flow is $\dot u_j = -\bigl(\partial \mathcal{L}_\rec/\partial u_j|^\mathrm{STE} + \beta \, \partial \mathcal{L}_\com/\partial u_j\bigr)$ and $\dot v_j = -\partial \mathcal{L}_\rec/\partial v_j$. Substituting gives
\begin{equation}
  \dot u_j \;=\; 2\sigma_j^2\, v_j (1 - c_j u_j v_j) \;-\; \frac{2\beta D_j}{u_j}, \qquad
  \dot v_j \;=\; 2\sigma_j^2\, c_j u_j (1 - u_j v_j),
\end{equation}
as stated in the main text.

\subsection{Water-level dynamics in the diagonal case}
\label{app:water-level-diagonal}

Here, we derive \eqref{eq:water-level-dynamics} from Section~\ref{sec:rdae-dynamics}, which characterizes how the reverse water-filling variance threshold rises throughout RD-AE training. As described in the main text, this is extremely consequential: it progressively decreases the active set of modes which are represented by the quantizer.

Given a rate budget $R$, a classical result~\citep[Ex.
I.9]{polyanskiy2025information} relates the rate-distortion water level $D^\star$ to the eigenvalues $\lambda_j$ via the constraint:
\begin{equation}
  \textstyle\sum_j \tfrac{1}{2} \log_2^+(\lambda_j/D^\star) \;=\; R,
  \label{eq:water-fill-app}
\end{equation}
We define the set of active modes to be $A := \{i: \lambda_i > D*\}$. The sum \eqref{eq:water-fill-app} then reduces to 
\begin{equation}
  \sum_{i \in A} \tfrac{1}{2} \log_2(\lambda_i / D^\star) \;=\; R.
  \label{eq:water-fill-app}
\end{equation}
We derive its dependence on $u_i$ (equivalently, on $\lambda_i = u_i^2 \sigma_i^2$) by implicit differentiation.

Rewriting~\eqref{eq:water-fill-app} as $\sum_{i \in A} \log_2 \lambda_i - k \log_2 D^\star = 2R$, where $k = |A|$, and differentiating in $u_\ell$ for $\ell \in A$:
\begin{equation}
  \frac{1}{\lambda_\ell \ln 2} \cdot \frac{\partial \lambda_\ell}{\partial u_\ell} \;-\; \frac{k}{D^\star \ln 2} \cdot \frac{\partial D^\star}{\partial u_\ell} \;=\; 0.
\end{equation}
Using $\partial \lambda_\ell / \partial u_\ell = 2 u_\ell \sigma_\ell^2 = 2 \lambda_\ell / u_\ell$, this gives
\begin{equation}
  \frac{\partial D^\star}{\partial u_\ell} \;=\; \frac{2 D^\star}{k\, u_\ell} \qquad \text{for } \ell \in A.
  \label{eq:water-level-static}
\end{equation}
For $\ell \notin A$, mode $\ell$ does not enter~\eqref{eq:water-fill-app} and $\partial D^\star / \partial u_\ell = 0$.

The time-derivative form follows by the chain rule along a flow trajectory:
\begin{equation}
  \frac{\dot D^\star}{D^\star} \;=\; \sum_{\ell \in A} \frac{1}{D^\star} \frac{\partial D^\star}{\partial u_\ell} \dot u_\ell \;=\; \frac{2}{k}\sum_{\ell \in A} \frac{\dot u_\ell}{u_\ell}.
  \label{eq:water-level-dynamic}
\end{equation}

As active modes grow, $D^\star$ rises via \eqref{eq:water-fill-app}. When $D^\star$ reaches $\lambda_j$ for the weakest active mode, that mode transitions to inactive. Once inactive, it remains so: at $\beta = 1$, inactive modes are frozen ($\dot{r}_j = 0$ from \eqref{eq:rdae-logistic}), so $\lambda_j$ cannot grow back above $D^\star$.

\subsection{Per-mode reduction at $\beta = 1$ and balanced init}
\label{app:beta1-reduction}

We specialize the flow to $\beta = 1$ and balanced initialization $u_j(0) = v_j(0) = \varepsilon$. We show that (i) balance is preserved on active modes, (ii) the active-mode activation $r_j = u_j v_j = u_j^2$ obeys a Saxe-like logistic, and (iii) inactive modes at balance are frozen.

\paragraph{Balance preservation on active modes.}
From the flow with $\beta = 1$ and $j \in A$:
\begin{equation}
  \dot u_j - \dot v_j \;=\; 2\sigma_j^2 v_j(1 - c_j u_j v_j) - \frac{2 D_j}{u_j} - 2\sigma_j^2 c_j u_j (1 - u_j v_j).
\end{equation}
Evaluating at $u_j = v_j =: a_j$ and using $D_j = D^\star = (1 - c_j)\lambda_j = (1 - c_j) a_j^2 \sigma_j^2$ on active modes,
\begin{equation}
  \dot u_j - \dot v_j\big|_{u_j = v_j}
  \;=\; 2\sigma_j^2 a_j (1 - c_j a_j^2) - 2\sigma_j^2 (1 - c_j) a_j - 2\sigma_j^2 c_j a_j (1 - a_j^2)
  \;=\; 0.
\end{equation}
Hence $u_j(t) = v_j(t)$ for all $t$ on active modes; write $a_j(t)$ for this common value.

\paragraph{Active-mode logistic.}
At balance, $\dot u_j = \dot v_j$, so we compute either. From $\dot v_j = 2\sigma_j^2 c_j u_j(1 - u_j v_j)$ with $u_j = v_j = a_j$:
\begin{equation}
  \dot a_j \;=\; 2\sigma_j^2 c_j a_j (1 - a_j^2).
\end{equation}
With $r_j = a_j^2$, $\dot r_j = 2 a_j \dot a_j$, giving
\begin{equation}
  \dot r_j \;=\; 4\sigma_j^2 c_j \, r_j(1 - r_j) \qquad \text{for } j \in A,
\end{equation}
which is Saxe's logistic~\eqref{eq:ae-sigmoid} with timescale $1/(4\sigma_j^2)$ stretched to $1/(4\sigma_j^2 c_j)$ by the channel gain.

\paragraph{Inactive modes.}
For $j \notin A$, $c_j = \tau_j^2 = 0$ and $D_j = \lambda_j = u_j^2 \sigma_j^2$. The flow reduces to
\begin{equation}
  \dot u_j \;=\; 2\sigma_j^2 v_j - \frac{2 u_j^2 \sigma_j^2}{u_j} \;=\; 2\sigma_j^2(v_j - u_j), \qquad
  \dot v_j \;=\; 0.
\end{equation}
At balance $u_j = v_j$, $\dot u_j = 0$ and $\dot v_j = 0$, so inactive modes at balance are stationary.

\subsection{Plain-AE mode activation (closed form)}
\label{app:ae-closed-form}

We derive the closed-form solution of the plain-AE flow from \citet{saxe2014exact} for completeness and reference in the warm-up theorem proof. The plain-AE gradient flow from the decoupled loss is
\begin{equation}
  \dot u_j \;=\; 2\sigma_j^2 v_j (1 - u_j v_j), \qquad \dot v_j \;=\; 2\sigma_j^2 u_j (1 - u_j v_j).
\end{equation}
From $\dot u_j - \dot v_j = -2\sigma_j^2 (1 - u_j v_j)(u_j - v_j)$, the balance gap $u_j - v_j$ decays to zero from any initial condition; in particular, balanced initialization $u_j(0) = v_j(0) = \varepsilon$ is preserved. Writing $a_j(t)$ for the common value and $r_j = a_j^2$,
\begin{equation}
  \dot a_j \;=\; 2\sigma_j^2 a_j (1 - a_j^2),
  \qquad
  \dot r_j \;=\; 4\sigma_j^2 r_j(1 - r_j).
\end{equation}
Separating variables, $\mathrm{d}r_j / [r_j(1 - r_j)] = 4\sigma_j^2 \, \mathrm{d}t$, and integrating from $r_j(0) = \varepsilon^2$ gives
\begin{equation}
  r_j(t) \;=\; \Bigl(1 + \tfrac{1 - \varepsilon^2}{\varepsilon^2} e^{-4\sigma_j^2 t}\Bigr)^{-1}.
\end{equation}

\subsection{Proof of Theorem~\ref{thm:warmup}}
\label{app:warmup}

The proof has three steps: (i) identify the active set at the switch time, (ii) show warmup-inactive modes stay frozen, (iii) derive the loss bound.

\paragraph{State at the switch.}
Plain-AE training from balanced initialization $W_1(0) = W_2(0) = \varepsilon I$ preserves diagonality and balance (\S\ref{app:ae-closed-form}), so at time $T_\mathrm{wu}$ the weights are $W_1 = W_2 = \diag(\sqrt{g_j(T_\mathrm{wu})})$ and the latent eigenvalues are $\lambda_j(T_\mathrm{wu}) = g_j(T_\mathrm{wu})\,\sigma_j^2$. Since $\sigma_j^2$ is strictly decreasing in $j$ and $g_j(T_\mathrm{wu})$ is strictly decreasing in $j$ (larger $\sigma_j^2$ gives faster Saxe activation), $\lambda_j(T_\mathrm{wu})$ is strictly decreasing in $j$.

The active set at the switch is therefore a top prefix $\{1, \ldots, m\}$, where $m$ is the maximal integer satisfying the water-fill consistency condition $\lambda_m(T_\mathrm{wu}) > D^\star_0$, with $D^\star_0 = \bigl(\prod_{i \le m} \lambda_i(T_\mathrm{wu})\bigr)^{1/m} 2^{-2R/m}$. Rearranging yields $R > \tfrac{1}{2}\sum_{i<m}\log_2(\sigma_i^2 g_i / \sigma_m^2 g_m)$, so $m = m^\mathrm{wu}(T_\mathrm{wu}, R)$.

\paragraph{Warmup-inactive modes are frozen forever.}
For $j > m^\mathrm{wu}$, mode $j$ is inactive at the switch ($c_j = \tau_j^2 = 0$) and balanced. By \S\ref{app:beta1-reduction}, $\dot u_j = \dot v_j = 0$ at balance, so $\lambda_j$ is constant. The water level $D^\star$ is non-decreasing along the flow (\S\ref{app:water-level-diagonal}), so $\lambda_j < D^\star$ for all $t \ge T_\mathrm{wu}$, and mode $j$ remains inactive.

\paragraph{The active set shrinks as a prefix.}
Active modes have $\dot r_j = 4\sigma_j^2 c_j r_j(1 - r_j) \ge 0$ (\S\ref{app:beta1-reduction}), so $\lambda_j$ is non-decreasing on the active set. Combined with monotonic $D^\star$, modes can only exit the active set, not enter. Since $\lambda_j$ remains ordered (larger $\sigma_j^2$ drives faster growth, preserving the decreasing sequence), the active set at any later time is a prefix $\{1, \ldots, k\}$ with $k \le m^\mathrm{wu}$. Any surviving mode converges to $r_j = 1$, giving $A_\infty = \{1, \ldots, k_\infty\}$ with $k_\infty \le m^\mathrm{wu}$ and $\lambda_j = \sigma_j^2$ on $A_\infty$.

\paragraph{Loss bound.}
At final state, the reconstruction loss is
\begin{equation}
  \mathcal{L}_\rec^\infty
  \;=\; k_\infty\, D^\star_\infty \;+\; \sum_{j > k_\infty} \sigma_j^2
  \;=\; k_\infty \bigl(\textstyle\prod_{i \le k_\infty} \sigma_i^2\bigr)^{1/k_\infty} 2^{-2R/k_\infty} \;+\; \sum_{j > k_\infty} \sigma_j^2,
\end{equation}
since each surviving mode pays water-fill distortion $D^\star_\infty = \delta_{k_\infty}(R)$ and each inactive mode pays $\sigma_j^2$ (the decoder receives zero on inactive coordinates). This is the Shannon distortion of an idealized scheme that activates the top $k_\infty$ modes and ignores the rest. Since $k_\infty \le m^\mathrm{wu}$ and the Shannon optimal scheme for the top $m^\mathrm{wu}$ source at rate $R$ activates all $m^\mathrm{wu}$ modes (with distortion $m^\mathrm{wu} \delta_{m^\mathrm{wu}}(R) + \sum_{j > m^\mathrm{wu}} \sigma_j^2$), the $k_\infty$-mode scheme is suboptimal, and
\begin{equation}
  \mathcal{L}_\rec^\infty \;\ge\; m^\mathrm{wu} \, \delta_{m^\mathrm{wu}}(R) \;+\; \sum_{j > m^\mathrm{wu}} \sigma_j^2.
\end{equation}
\qed

\subsubsection{Numerical validation of Theorem~\ref{thm:warmup}}
\label{app:theorem-synthetic}

\begin{figure}[t]
  \centering
  \includegraphics[width=\linewidth]{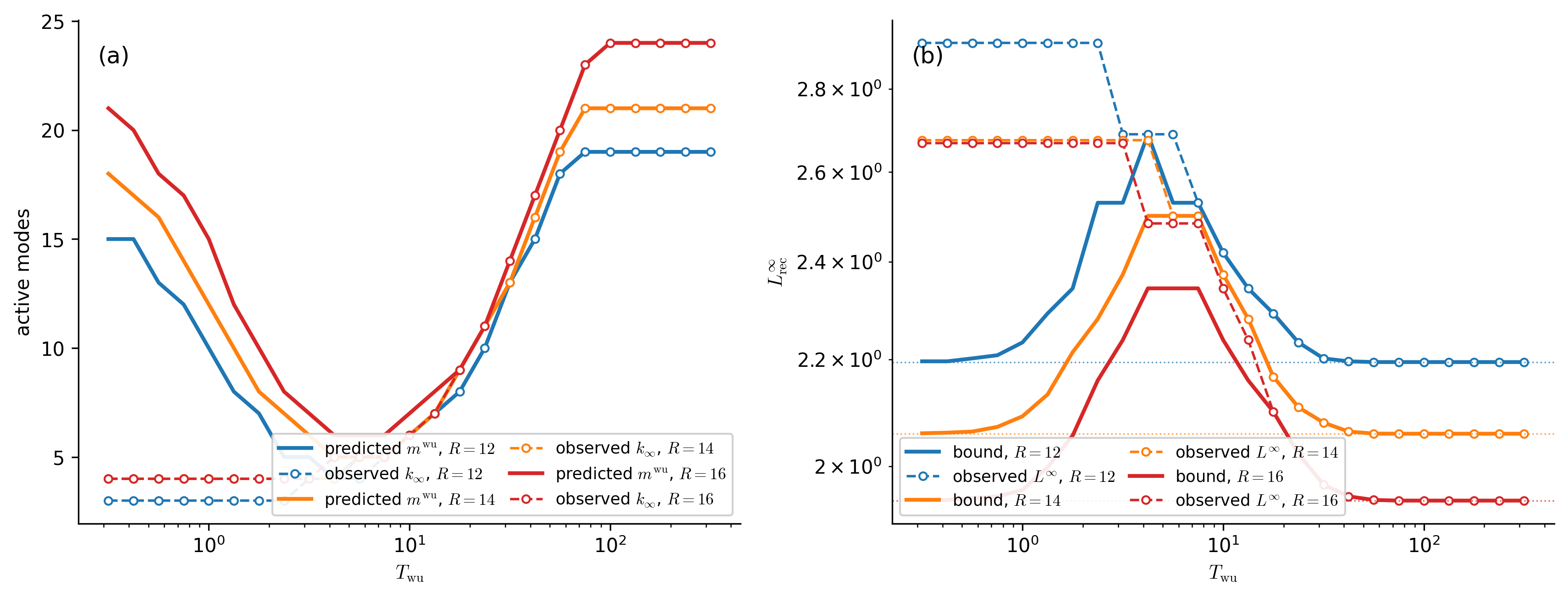}
  \caption{\textbf{Warmup theorem: predictions vs.\ simulation.} Balanced
    initialization $W_1(0) = W_2(0) = \varepsilon I$ with $\varepsilon = 0.1$,
    $\beta = 1$, $d = 64$, $\sigma_j^2 = j^{-1}$, three rates
    $R \in \{12, 14, 16\}$ bits. \textbf{(a)} Active-mode count at convergence.
    Solid lines: predicted $m^\mathrm{wu}(T_\mathrm{wu}, R)$
    from~\eqref{eq:m-wu}. Dashed lines: observed $k_\infty$ from simulated
    RD-AE flow. \textbf{(b)} Reconstruction loss. Solid lines: loss bound
    from~\eqref{eq:loss-bound}. Dashed lines: observed $\mathcal{L}_\rec^\infty$.
    Dotted horizontal lines mark the Shannon floor $D(R)$ for each rate. As
    $T_\mathrm{wu}$ grows, predictions and observations both approach the
    Shannon limit; at small $T_\mathrm{wu}$ the bound is vacuous and the
    observed collapse is severe. %
    }
  \label{fig:theory-warmup}
\end{figure}

Figure~\ref{fig:theory-warmup} compares the theorem to simulation across warm-up durations in the linear RD-AE setting. Left panel: $m^\mathrm{wu}(T_\mathrm{wu}, R)$ from~\eqref{eq:m-wu} against the observed final active-set size $k_\infty$ in RD-AE simulation, for $R \in \{12, 14, 16\}$ bits. Right panel: the loss bound~\eqref{eq:loss-bound} against the observed final $\mathcal{L}_\rec^\infty$, with the rate-distortion floor $D(R)$ as reference. The figure shows $T_\mathrm{wu}$ in the regime where the bound is informative; at $T_\mathrm{wu} = 0$ the bound reduces to $D(R)$ and the observed loss exceeds it. The bound tightens as $T_\mathrm{wu}$ grows, and both $k_\infty$ and $\mathcal{L}_\rec^\infty$ approach the Shannon limit.

\section{Experiment Details}
\label{app:experiment_details}

\input{figures/ae_activation_adam.tex}

\subsection{VQGAN Experiments}
\label{app:vqgan_experiment_detail}

\paragraph{Architecture and loss.}
We use the f16 VQGAN encoder--decoder from \citet{esser2021vqgan} with the \texttt{imagenet\_vqgan.yaml} configuration (128 base channels, multipliers $[1,1,2,2,4]$, two residual blocks per level, self-attention at $16\times 16$, 256-channel bottleneck, 256-dimensional codebook embedding space). At $128\times 128$ input resolution each image is tokenized to an $8\times 8$ grid of 64 discrete tokens. As per the original VQGAN training recipe, the discriminator is disabled for the first 250k training steps, the remaining loss is $\mathcal{L} = \mathcal{L}_{L_1} + \mathcal{L}_{\text{LPIPS}} + \lambda_{\text{cb}}\,\mathcal{L}_{\text{VQ}}$ with $\lambda_{\text{cb}} = 1.0$ and commitment weight $\beta = 0.25$, following \citet{esser2021vqgan}.

\paragraph{Training.}
All runs use the default VQGAN optimizer settings: Adam with $(\beta_1, \beta_2) = (0.5, 0.9)$, no weight decay, and constant effective learning rate $2.88\times 10^{-4}$. We emphasize that any ``warm-up'' referenced in this paper is the AE phase on the model, not an LR schedule. ImageNet-100 runs are distributed across $4\times$ A100 40\,GB GPUs at batch size 16/GPU (effective batch size 64), completing in approximately 5.5 hours of wall-clock time. ImageNet-1K runs are distributed across $2 \times$ H200 GPUs at batch size $8$/ GPU (effective batch size 16).

\paragraph{Codebook handling.}
\emph{Vanilla VQGAN} initializes the codebook from $\mathcal{U}(-1/K, 1/K)$ and updates it via straight-through gradients. \emph{VQGAN w/ Respawn} and all \emph{AE Warm-up} recipes use $k$-means initialization on encoder outputs (100 Lloyd iterations over the full training set) with EMA updates, and respawn any code whose EMA usage (decay 0.99) falls below 0.01 from a random encoder output in the current batch. %

\paragraph{Data.} Our ImageNet-1K experiments are performed at $256 \times 256$ resolution. Our ImageNet-100 experiments sample 20,000 training images uniformly at random from the 100-class subset standardized by \citet{tian2020contrastive}, and retain the entire 5,000-image validation set. ImageNet-100 images are center-cropped and resized to $128 \times 128$.

\paragraph{Evaluation.}
We report L1 and LPIPS (VGG-16 backbone) as the standard pixel and perceptual reconstruction terms; Rec Loss is their sum. CB Util\% is the fraction of codes receiving at least one assignment on the validation set, and CB Dim is the 99\%-variance PCA effective dimension of the codebook entries. rFID uses the InceptionV3 backbone. %

\subsection{WavTokenizer Experiments}
\label{app:wavtokenizer-experiments}

\paragraph{Architecture.}
We use the official WavTokenizer 75\,FPS architecture from \citet{ji2025wavtokenizer} with a 512-dimensional pre-quantization latent and a single codebook. Unless otherwise noted, all architectural and loss hyperparameters are inherited from the released 75\,FPS checkpoint configuration, including the convolutional encoder/decoder backbone, the multi-period and multi-scale STFT discriminators, and the full reconstruction loss suite (time-domain, mel-spectrogram, feature matching, adversarial). We vary only the codebook size $K \in \{4096, 8192, 16384, 65536\}$ and the training regime (cold-start vs.\ AE warm-up).

\paragraph{Training.}
All runs train for 100 epochs on $2\times$ H200 GPUs with a per-GPU batch size of 40; at 75\,FPS this corresponds to 9{,}000 latent vectors per GPU per batch. The optimizer, learning-rate schedule, discriminator schedule, and loss weightings follow the released WavTokenizer recipe for the 75FPS model with codebook size 4096. All runs use the same random seed (the WavTokenizer default); we did not run multiple seeds due to compute constraints--each run takes approximately 7 days on $2 \times$ H200 GPUs.

\paragraph{AE warm-up training.}
For warm-start runs, we first train the encoder and decoder as a plain autoencoder for 43 epochs, utilizing both reconstruction and discriminator loss. This single warm-up checkpoint is reused across all codebook sizes. At epoch 43, we enable quantization: we initialize the codebook via scikit-learn $k$-means on encoder outputs for $K \in \{4096, 8192, 16384\}$, and fall back to random initialization for $K = 65536$, where $k$-means becomes prohibitive at this codebook count. Training then proceeds for the remaining 57 epochs under the standard recipe with EMA codebook updates (decay 0.99) and respawn enabled. We did not sweep over warm-up duration due to compute constraints; 43 epochs was chosen as the point at which the latent effective dimension $d_{\text{eff}}$ visibly plateaus in Figure~\ref{fig:ae_activation_wav}, matching the stopping rule advocated in Section~\ref{sec:experiments}.

\paragraph{Codebook respawn.}
We respawn codes whose EMA usage falls below a codebook-size-dependent threshold, aggregating usage counts across both GPUs each epoch for stability at large $K$. Cold-start baselines use the WavTokenizer default threshold of 1.0. Warm-start runs scale the threshold inversely with $K$ to keep the expected number of respawned codes per epoch roughly constant:
\begin{center}
\begin{tabular}{cc}
\toprule
Codebook size $K$ & Dead-code threshold \\
\midrule
4{,}096  & 1.0 \\
8{,}192  & 0.5 \\
16{,}384 & 0.25 \\
65{,}536 & 0.06 \\
\bottomrule
\end{tabular}
\end{center}

\paragraph{Data.}
We train on both train-clean-360 and train-other-500 sampled at 24\,kHz, with audio clips of 3s, following the original WavTokenizer training recipe \citep{ji2025wavtokenizer}. Validation and model selection use 100 randomly sampled batches of combined dev-clean and dev-other; final evaluation reports metrics on both LibriTTS test-clean and test-other. For each training run, we select the checkpoint with the best validation loss.

\paragraph{Evaluation.}
We report five reconstruction metrics on test-clean and test-other: mel-spectrogram L1 loss (Mel), UTMOS \citep{saeki2022utmos}, PESQ \citep{rix2001pesq}, STOI \citep{taal2010stoi}, and voiced/unvoiced F1 (V/UV F1) computed following the methodology of \citet{morrison2022chunked}. With the exception of  L1 loss, these match the metrics reported in \citet{ji2025wavtokenizer}. Codebook utilization is computed over the full evaluation set, and codebook effective dimension $\dim(\mathcal{C})$ is the number of PCA components needed to explain 99\% of the variance across codebook entries, matching the definition used for VQGAN in Appendix~\ref{app:vqgan_experiments}.

\section{Additional results on VQGAN}
\label{app:vqgan_experiments}

\input{tables/vqgan_1k_500k}
\cref{tab:vqgan_1k_500k} shows the final codebook utilization, $d_\text{eff}$, L1, LPIPS, and rFID for our ImageNet-1K experiments.

\input{figures/vqgan_100k}

\input{figures/vqgan_warmup_100k_scaling.tex}
\cref{fig:vqgan_100k} shows the training trajectories for the experiments across all codebook sizes and warm-up runs, and \cref{fig:warmup-scaling} shows the effect of the number of AE warm-up steps on final reconstruction loss after 100k total training steps. We observe that for all codebook sizes, the effect of adding warm-up steps plateaus at around 20k steps, and increasing warm-up duration past this point does not improve final loss.

\subsection{Water-filling predicts final VQ-VAE codebook dimension}

\cref{thm:warmup} bounds the surviving codebook dimension by $m_{\text{wu}}(T_{\text{wu}}, R)$, the number of latent PCA components that lie above the Shannon water level at rate $R = \log_2 K$. The theorem is stated for the linear-Gaussian RD-AE; here we test whether this holds for VQGAN.

\begin{figure}[t]
    \centering
    \begin{minipage}[t]{0.33\linewidth}
        \centering
        \includegraphics[width=\linewidth]{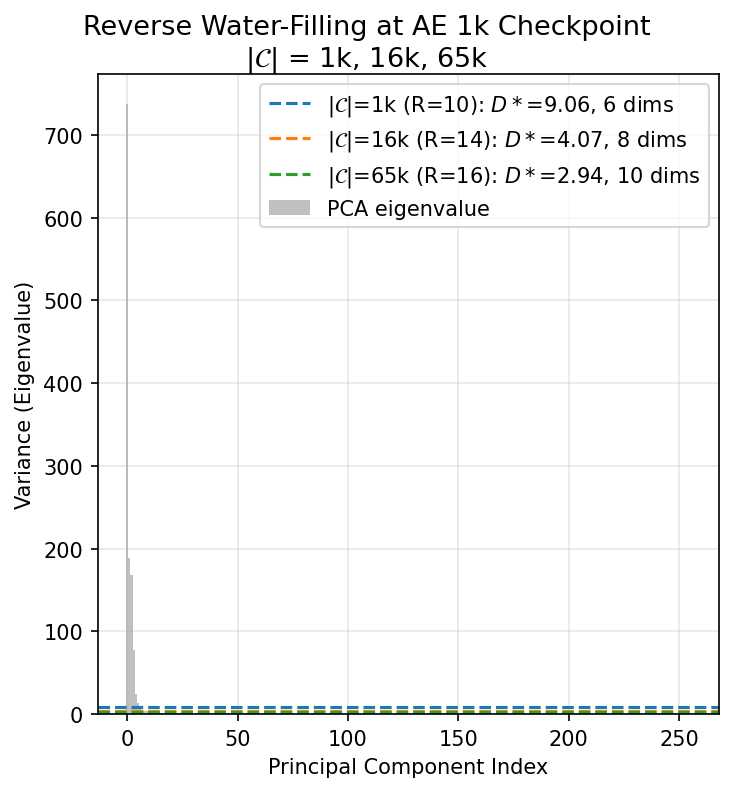}
    \end{minipage}\hfill
    \begin{minipage}[t]{0.33\linewidth}
        \centering
        \includegraphics[width=\linewidth]{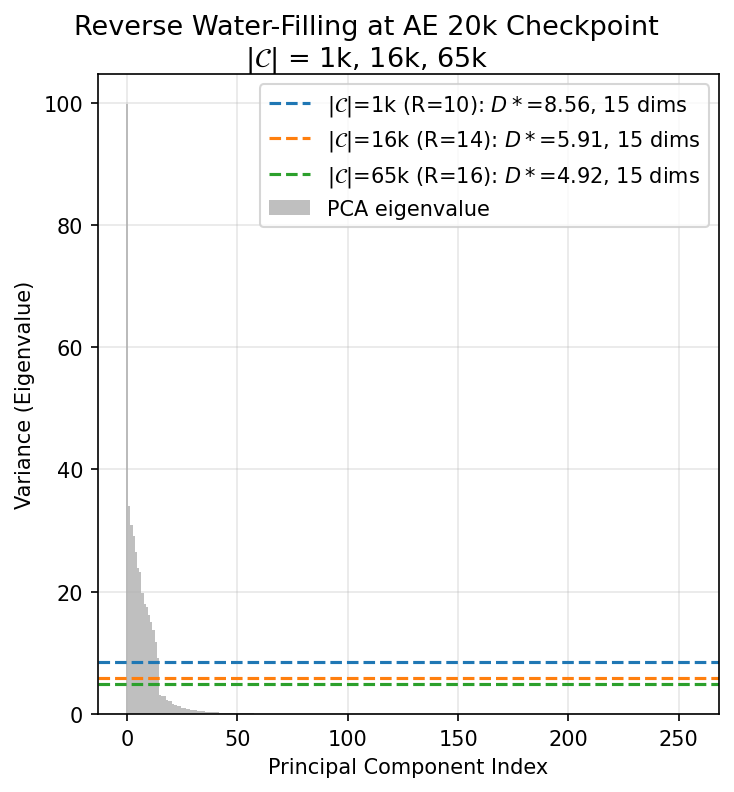}
    \end{minipage}\hfill
    \begin{minipage}[t]{0.33\linewidth}
        \centering
        \includegraphics[width=\linewidth]{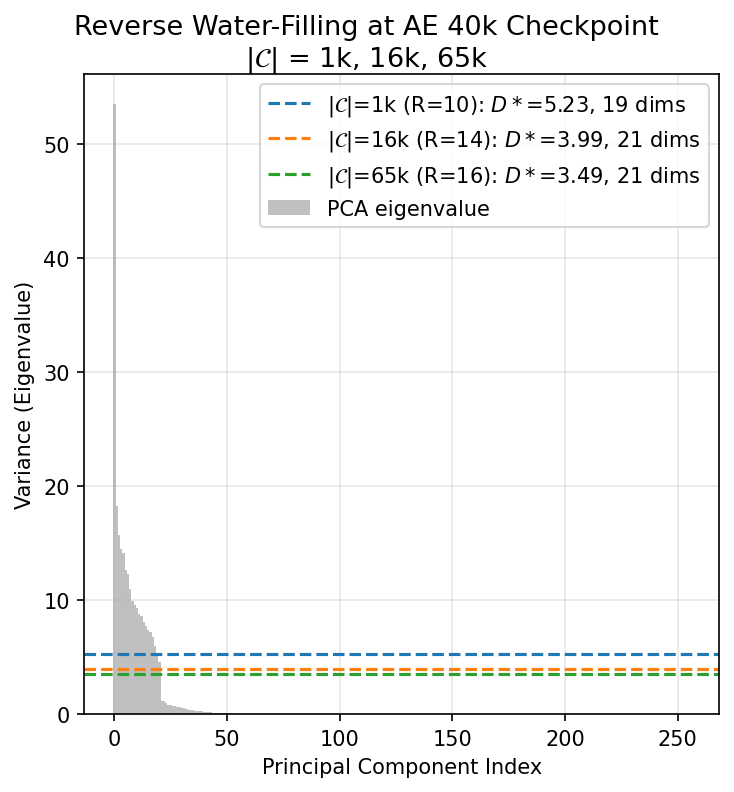}
    \end{minipage}
    \caption{\textbf{AE latent PCA spectrum and water levels at three warmup durations.} Gray bars: PCA eigenvalues of the pre-quantization latent at AE warmup steps 1k, 20k, and 40k (left to right), computed over the full ImageNet-100 (20k) training set. Dashed horizontal lines: Shannon water level $D^\star$ at rates $R = \log_2 K$ for codebook sizes $K \in \{2^{10}, 2^{14}, 2^{16}\}$, computed by reverse water-filling (Eq.~\ref{eq:water-fill}) on the displayed spectrum. The number of PCA components above each line is the predicted active-mode count $m_{\text{wu}}(T_{\text{wu}}, R)$ from \cref{thm:warmup}. Two trends visible across panels: as $T_{\text{wu}}$ grows, more modes lift above each line (weak modes activate during AE training); at fixed $T_{\text{wu}}$, larger $K$ pushes the water line lower and admits more modes.}
    \label{fig:ae-variances-waterfilling-all-cb-pca}
\end{figure}

\begin{figure}[t]
    \centering
    \includegraphics[width=1.0\linewidth]{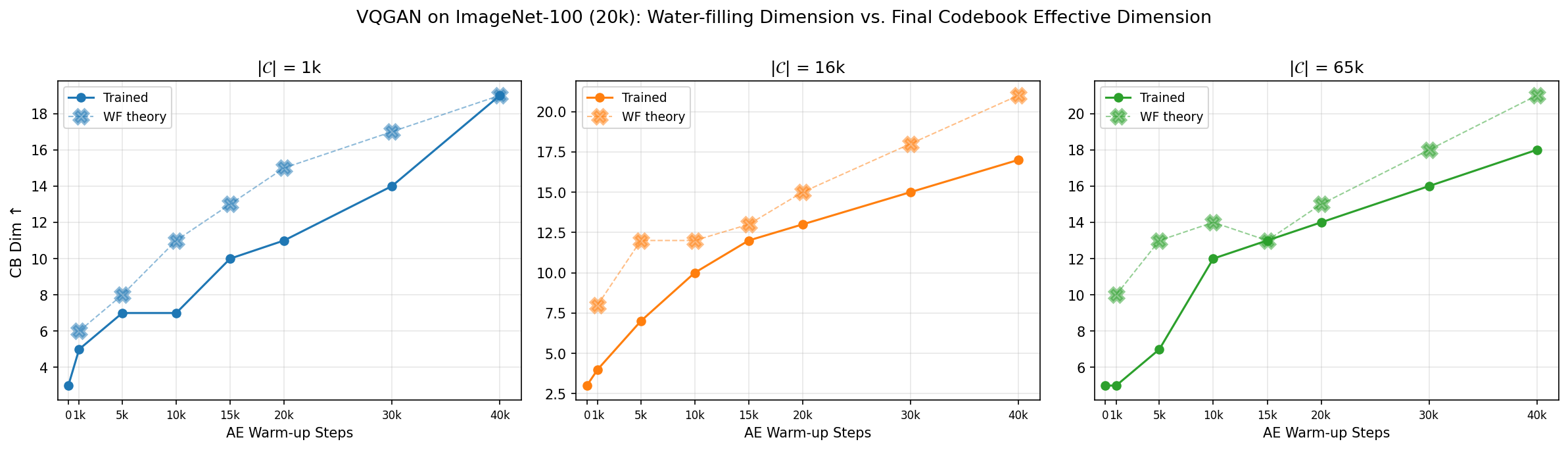}
    \caption{\textbf{Water-filling on the AE latent spectrum upper-bounds the trained VQ-VAE codebook dimension.} For each AE warm-up checkpoint $T_{\text{wu}} \in \{0, 1\text{k}, 5\text{k}, 10\text{k}, 15\text{k}, 20\text{k}, 30\text{k}, 40\text{k}\}$ and each codebook size $K \in \{2^{10}, 2^{14}, 2^{16}\}$ (one panel per $K$), we compare the water-filling prediction $m_{\text{wu}}(T_{\text{wu}}, R)$ from the AE PCA spectrum (dashed lines) against the codebook effective dimension of the corresponding VQ-VAE trained from that checkpoint to 100k total steps (solid lines). The predicted dimension lies above the trained dimension at every point and across all three codebook sizes: water-filling on the AE latent is an empirical upper bound on what the downstream VQ-VAE achieves. Both quantities increase monotonically with $T_{\text{wu}}$, and the gap between them is roughly constant.}
    \label{fig:vqgan-waterfilling-ub}
\end{figure}

\paragraph{Method.} For each AE warm-up checkpoint, we compute the PCA spectrum of the pre-quantization latent over the full ImageNet-100 (20k) training set. We then compute reverse water-filling at rates $R \in \{10, 14, 16\}$ bits, matching the codebook sizes used in \cref{tab:vqgan_ablation}. The number of PCA components whose eigenvalue exceeds $D^\star$ gives the predicted active-mode count $m_{\text{wu}}$. \cref{fig:ae-variances-waterfilling-all-cb-pca} visualizes the spectrum and water lines at three checkpoints; \cref{fig:vqgan-waterfilling-ub} compares $m_{\text{wu}}$ to the trained VQ-VAE codebook effective dimension at every $(T_{\text{wu}}, K)$ pair.

\paragraph{Key findings.} The water-filling prediction sits strictly above the trained codebook dimension in all 24 settings, validating the theorem's bound qualitatively in the nonlinear regime. As predicted, the predicted dimension grows with $T_{\text{wu}}$ (as more weak modes are learned during AE training) and with $K$ (higher rate lowers the water line, admitting more modes). Both trends are reproduced by the trained VQ-VAEs.

\subsection{VQGAN with $\beta=0$}
\label{app:vqgan_beta0}
\input{tables/vqgan_100k_withbeta0}

\begin{figure}[t]
    \centering
    \includegraphics[width=1.0\linewidth]{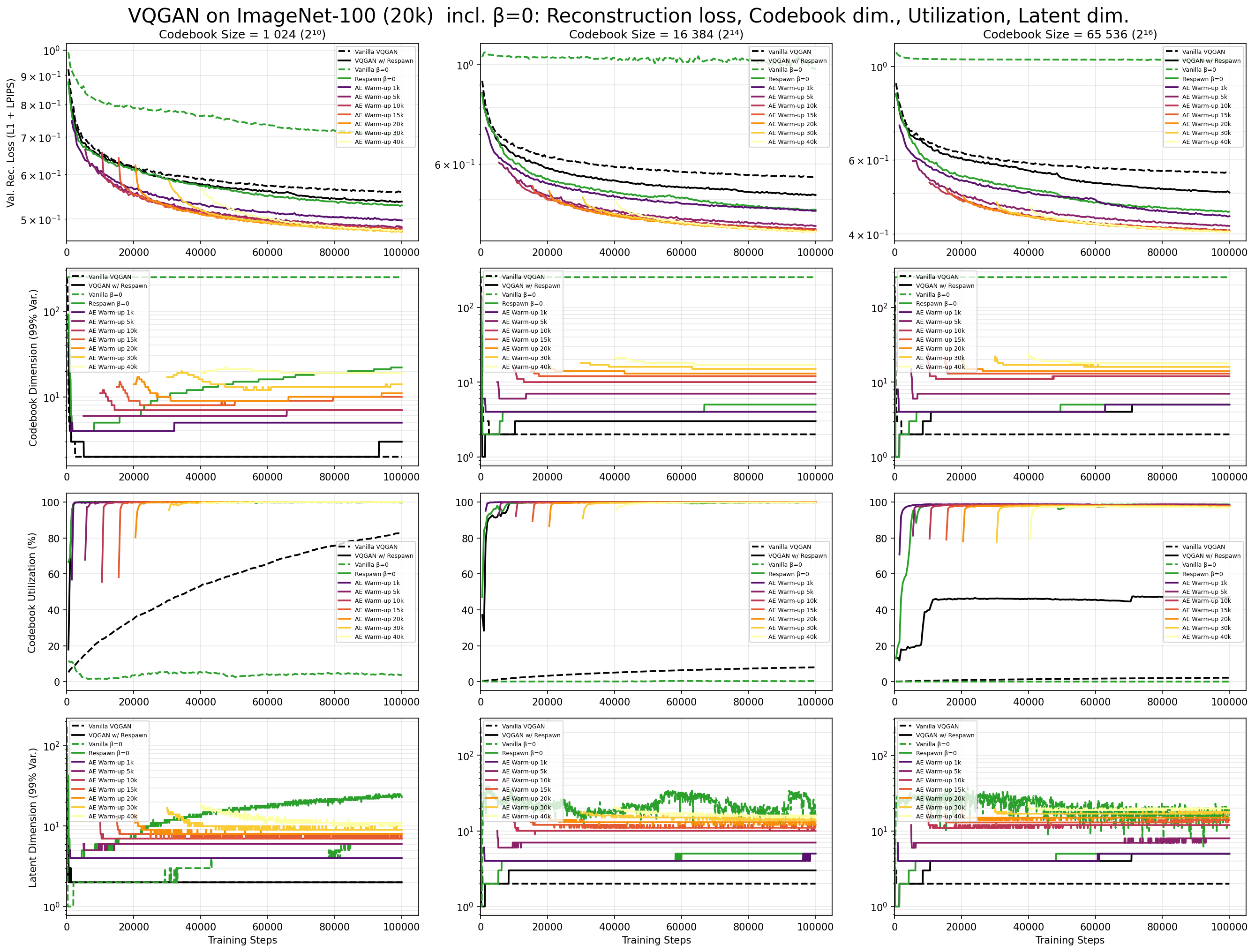}
    \caption{\textbf{VQGAN on ImageNet-100 (20k) with the commitment term removed.} Training curves at $\beta = 0$ (vanilla and respawn variants, dashed) overlaid on the standard $\beta > 0$ runs from \cref{fig:vqgan_100k}. Rows, top to bottom: validation reconstruction loss (L1 + LPIPS), codebook effective dimension, codebook utilization, latent effective dimension. Columns: codebook sizes $K \in \{2^{10}, 2^{14}, 2^{16}\}$. Vanilla VQGAN at $\beta = 0$ collapses to near-zero codebook utilization at every $K$ — almost all codewords are dead from initialization onward, since without the commitment term pulling the encoder toward the codebook, randomly initialized codes are never assigned. Respawn at $\beta = 0$ keeps codes alive and admits a higher latent effective dimension at small $K$ (consistent with the linear-theory escape route in \cref{sec:rdae-dynamics}), but does not translate to a lower reconstruction floor.}
    \label{fig:vgqan_ablation_beta0}
\end{figure}

The linear analysis in \cref{sec:rdae} identified $\beta = 0$ as a degenerate case in which inactive modes are not strictly frozen: at $\beta = 0$, the encoder ODE reads $\dot{u}_j = 2\sigma_j^2 v_j$, which is positive for any $v_j > 0$, so inactive modes can in principle grow back into the active set. \cref{fig:vgqan_ablation_beta0} and \cref{tab:vqgan_ablation_beta0} test this directly by removing the commitment term from VQGAN.

\paragraph{Vanilla VQGAN at $\beta = 0$.} Without commitment loss, codebook utilization drops to 3.7\%, 0.3\%, and 0.0\% at $K = 2^{10}, 2^{14}, 2^{16}$ respectively. The codebook effective dimension is reported as 248--254, but this is a degenerate artifact: random codebook initialization spreads codewords across the embedding space, and with no commitment term to pull encoder outputs toward existing codes, almost no codes are ever assigned. The unassigned codes remain at their random init values, inflating the PCA span without improving reconstruction.

\paragraph{Respawn at $\beta=0$.} With $k$-means initialization and dead-code respawn, $\beta=0$ raises the codebook dimension at small $K$ (22 vs.\ 3 at $K=2^{10}$), as the linear analysis predicts, but at larger $K$ it stays at ${\sim}5$, like the $\beta=0.25$ baseline. Its reconstruction relative to that baseline is mixed (worse L1 and rFID at $K=2^{10}$, better at $K\ge2^{14}$), and performs worse than AE Warm-up at every codebook size.

\paragraph{Role of the commitment loss.} The commitment term is intended to stabilize the encoder--codebook relationship by anchoring encoder outputs to codeword positions. However, as we show in Section~\ref{sec:rdae}, this loss term suppresses weak modes and leads to domension collapse. Removing the commitment loss prevents the suppression but breaks the stabilization, leading to higher codebook dimension but significantly worse reconstruction loss and codebook utilization, as shown in \cref{tab:vqgan_ablation_beta0}. AE warm-up sidesteps this trade-off entirely: weak modes activate \emph{before} VQ is introduced, and we see that training with warm-up outperforms $\beta = 0$ across every codebook size.

\section{Additional results on WavTokenizer}

\begin{figure}[t]
\centering
\includegraphics[width=1.0\textwidth]{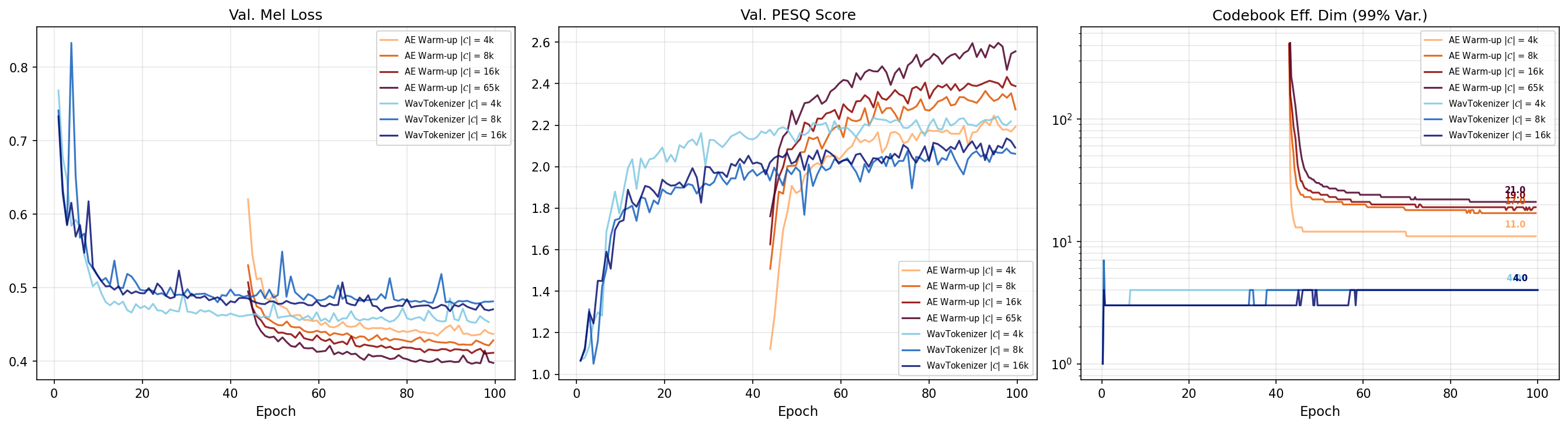}
\caption{\textbf{WavTokenizer on LibriTTS: cold-start training fails to benefit from larger codebooks; AE Warm-up restores codebook scaling.}
Cold-start WavTokenizer (blue curves, $\nK \in \{4\text{k}, 8\text{k}, 16\text{k}\}$) versus AE-VQ (warm curves, $\nK \in \{4\text{k}, 8\text{k}, 16\text{k}, 65\text{k}\}$).
\emph{From left:} validation mel loss, PESQ, codebook $d_\text{eff}$ (99\% variance).
AE Warm-up runs are trained for 43 epochs before introducing VQ.
Cold-start runs show flat loss and low effective dimension regardless of $\nK$; AE-VQ runs show monotone improvement with $\nK$ and substantially higher effective dimension in both latent and codebook spaces.}
\label{fig:wavtokenizer-curves}
\end{figure}

Figure~\ref{fig:wavtokenizer-curves} shows training trajectories for the WavTokenizer runs reported in Table~\ref{tab:wavtok_comparison}. 

We notice that, when switching from autoencoder training to VQ training at epoch 43, PESQ sharply degrades, but quickly recovers and surpasses the no-warm-up runs.

Cold-start runs plateau in Mel loss quickly, after the first 20--30 epochs. AE Warm-up runs continue to improve Mel loss after the quantizer is enabled; the additional training budget appears productive only when the latent has enough effective dimension.

%% file: figures/ae_activation_adam.tex
\begin{figure}[t]
\centering
\begin{subfigure}{0.49\textwidth}
    \centering
    \includegraphics[width=0.85\textwidth]{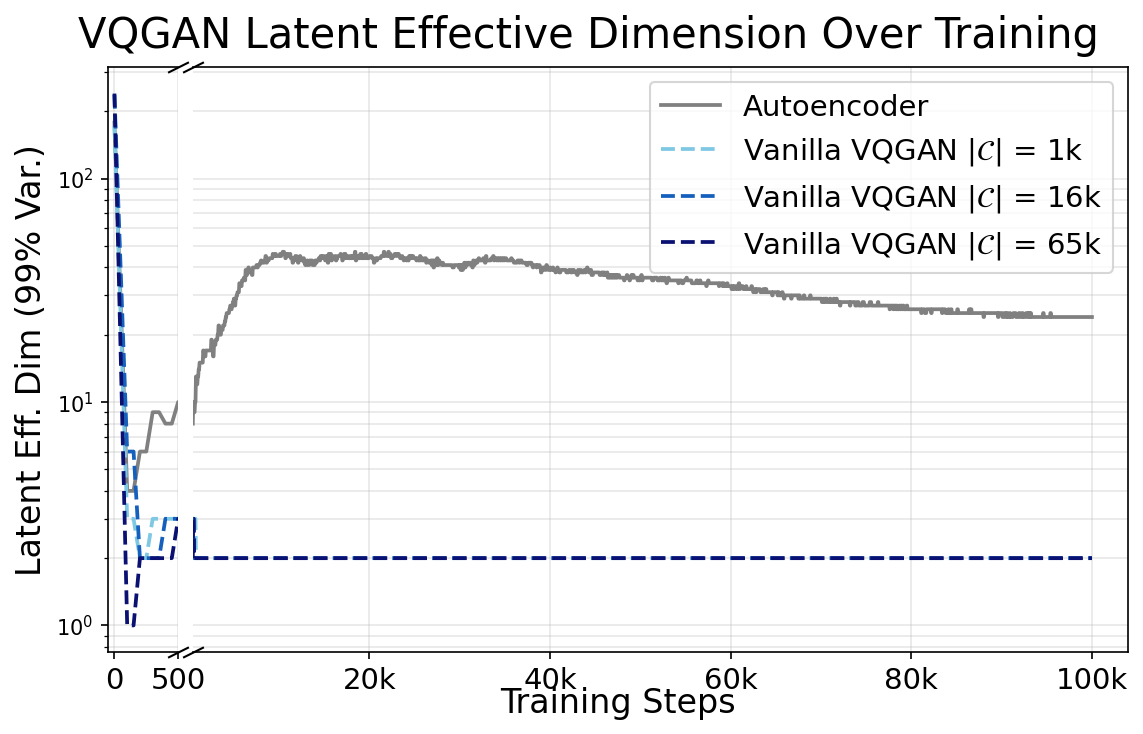}
    \caption{VQGAN on ImageNet-100 (20k)}
    \label{fig:ae_activation_vqgan}
\end{subfigure}
\hfill
\begin{subfigure}{0.49\textwidth}
    \centering
    \includegraphics[width=0.88\textwidth]{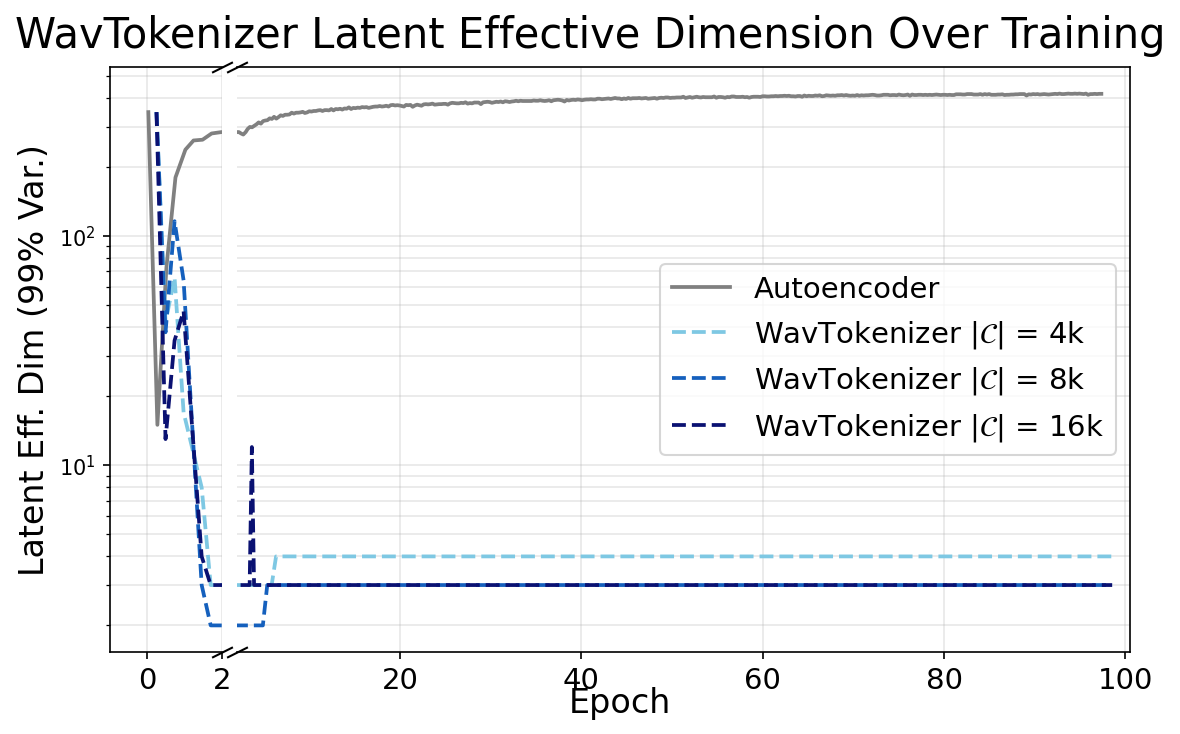}
    \caption{WavTokenizer on LibriTTS.}
    \label{fig:ae_activation_wav}
\end{subfigure}
\caption{\textbf{Without VQ, real architectures exhibit Saxe-like sequential
activation when trained with Adam. VQ causes dimensional collapse.} Latent effective dimension ($d_\mathrm{eff}$, $99\%$-variance
threshold) during autoencoder and standard VQ-VAE training. In both (a) VQGAN and (b) WavTokenizer, standard VQ training pins latent dimension at 2-4, while dimension rises as modes are learned sequentially when training as an autoencoder.
}
\label{fig:ae_activation}
\end{figure}

%% file: tables/vqgan_1k_500k.tex
\begin{table}[!t]
\centering
\setlength{\tabcolsep}{5pt}
\caption{\textbf{VQGAN on ImageNet-1K with $|\mathcal{C}| = 2^{14}$.}
L1, LPIPS, and rFID: lower is better. D: codebook $d_\text{eff}$ (higher is better).
Vanilla VQGAN uses uniform codebook initialization and no respawn; VQGAN w/ Respawn adds k-means initialization and respawn;
$T_\text{wu}$ additionally pre-trains the encoder--decoder as a plain autoencoder for the indicated number of steps before introducing the codebook.
Bold indicates the best value among VQ models.}
\label{tab:vqgan_1k_500k}
\begin{tabular}{lrrrrr}
\toprule
Variant
& Util\%
& $D$$\uparrow$
& L1$\downarrow$
& LPIPS$\downarrow$
& rFID$\downarrow$ \\
\midrule

Vanilla VQGAN
& 8.1
& 4
& 0.143
& 0.304
& 7.91 \\

\qquad w/ Respawn
& 100.0
& 5
& 0.136
& 0.284
& 7.07 \\

$T_\text{wu} = $ 1k
& 100.0
& 5
& 0.131
& 0.280
& 7.22 \\

$T_\text{wu} = $ 6k
& 100.0
& 11
& 0.136
& 0.268
& 5.97 \\

$T_\text{wu} = $ 10k
& 100.0
& 13
& \textbf{0.127}
& \textbf{0.260}
& 5.88 \\

$T_\text{wu} = $ 20k
& 100.0
& 16
& 0.131
& \textbf{0.260}
& \textbf{5.35} \\

$T_\text{wu} = $ 30k
& 100.0
& \textbf{18}
& 0.129
& \textbf{0.260}
& 5.37 \\
\bottomrule
\end{tabular}

\label{tab:vqgan_imagenet1k}
\end{table}

%% file: figures/vqgan_100k.tex
\begin{figure}[t]
\centering
\includegraphics[width=\textwidth]{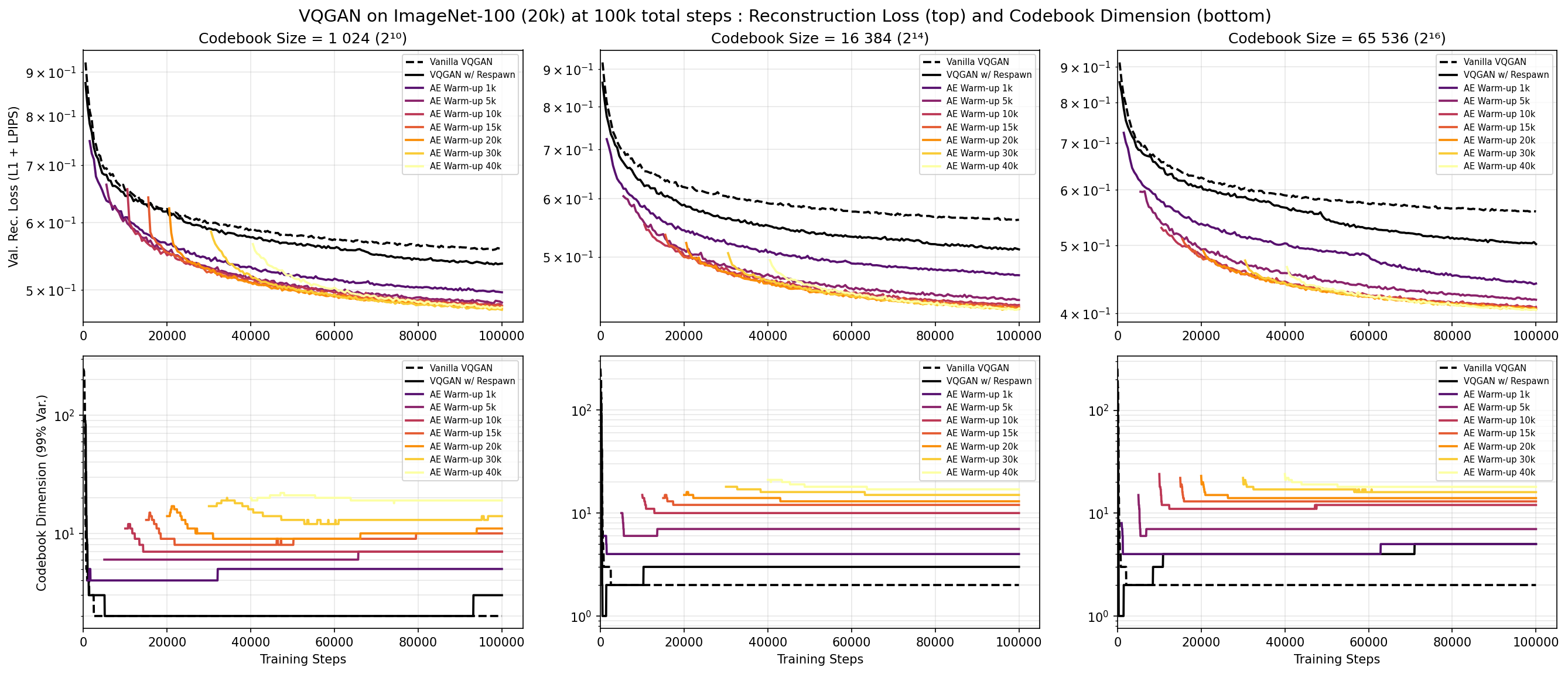}
\caption{\textbf{VQGAN on ImageNet-100: codebook size $\times$ warm-up duration.}
Training runs spanning $\nK \in \{1\text{k},\, 16\text{k}, 65\text{k}\}$ and $\Twu \in \{0,\, 1\text{k},\, 5\text{k},\, 10\text{k},\, 20\text{k},\, 30\text{k},\, 40\text{k}\}$ steps of AE Warm-up, all with k-means init and respawning of the codebook (solid lines), compared with the vanilla VQGAN training recipe with random codebook init and no respawn (dashed line).
\emph{Top row:} validation reconstruction loss
(L1 + LPIPS) during training. \emph{Bottom row:} codebook effective
dimension $d_{\mathrm{eff}}$ ($99\%$ variance threshold) during training. Vanilla VQGAN collapses to $d_{\mathrm{eff}} = 2$ and hits a $K$-independent loss floor
regardless of codebook size. Adding dead code respawn (No Warm-up) partially rescues
codebook utilization and modestly raises $d_{\mathrm{eff}}$, but still
plateaus in validation loss well above the warm-up runs. Additional AE warm-up steps
further increases the codebook effective dimension and lowers the
reconstruction floor.
}
\label{fig:vqgan_100k}
\end{figure}

%% file: figures/vqgan_warmup_100k_scaling.tex
\begin{figure}[t]
  \centering
  \includegraphics[width=1.0\linewidth]{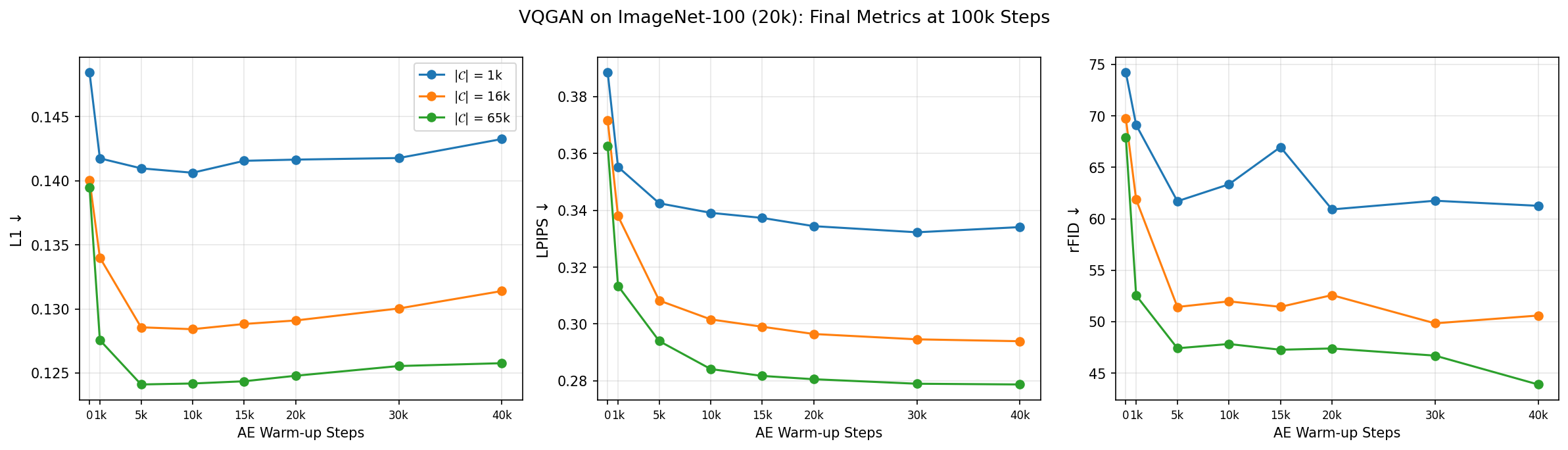}
  \caption{\textbf{AE warm-up vs. codebook-size scaling in VQGAN.}
  Final reconstruction metrics on the ImageNet-100 validation set at step
  $100{,}000$---L1 (left), LPIPS (center), and rFID (right)---as a function of
  AE warm-up length
  $T_{\mathrm{wu}} \in \{0, 1\text{k}, 5\text{k}, 10\text{k}, 15\text{k}, 20\text{k}, 30\text{k}, 40\text{k}\}$,
  for three codebook sizes $K \in \{2^{10}, 2^{14}, 2^{16}\}$.
  $T_{\mathrm{wu}} = 0$ corresponds to the VQGAN w/ Respawn baseline
  ($k$-means initialization plus dead-code respawn, no AE phase).
  At $T_{\mathrm{wu}} = 0$ the three curves are tightly clustered: increasing
  $K$ by $64\times$ yields only a modest reduction in any of the three metrics.
  As $T_{\mathrm{wu}}$ grows, the curves separate into the expected ordering
  $K = 2^{16} < 2^{14} < 2^{10}$ and the gap between them widens, indicating
  that warm-up not only lowers the reconstruction floor but also restores the
  scaling benefit of a larger codebook. Returns diminish past
  $T_{\mathrm{wu}} \approx 20\text{k}$, matching the plateau of the AE-phase
  effective-dimension trajectory in Figure~\ref{fig:ae_activation}.}
  \label{fig:warmup-scaling}
\end{figure}

%% file: tables/vqgan_100k_withbeta0.tex
\begin{table}[!t]
\centering
\scriptsize
\setlength{\tabcolsep}{2pt}

\caption{\textbf{ImageNet-100 (20k) warm-up ablation, with $\beta=0$ baselines.}
L1, LPIPS, and rFID: lower is better. Dim: codebook $d_\text{eff}$ (higher = more of the codebook span is used).
Bold indicates the best value within each codebook size.
Vanilla VQGAN uses uniform codebook init and no respawn; VQGAN w/ Respawn adds k-means init and respawn;
$T_\text{wu} = $ \{1k, \ldots, 40k\} additionally pre-trains the encoder--decoder as a plain autoencoder for the indicated number of steps before introducing the codebook.
The $\beta=0$ rows remove the commitment loss; without it, vanilla VQGAN collapses to near-zero codebook utilization while the encoder is free to span many latent dimensions, and respawn alone partially compensates for utilization but still falls short of warm-up at large $K$.
All runs share the same 100k-step budget, so warm-up trades VQ steps for AE steps.}
\vspace{1em}

\begin{tabular}{l *{15}{r}}
\toprule
& \multicolumn{5}{c}{$|\mathcal{C}| = 2^{10}$}
& \multicolumn{5}{c}{$|\mathcal{C}| = 2^{14}$}
& \multicolumn{5}{c}{$|\mathcal{C}| = 2^{16}$} \\
\cmidrule(lr){2-6} \cmidrule(lr){7-11} \cmidrule(lr){12-16}
Recipe
& Util\% & Dim$\uparrow$ & L1$\downarrow$ & LPIPS$\downarrow$ & rFID$\downarrow$
& Util\% & Dim$\uparrow$ & L1$\downarrow$ & LPIPS$\downarrow$ & rFID$\downarrow$
& Util\% & Dim$\uparrow$ & L1$\downarrow$ & LPIPS$\downarrow$ & rFID$\downarrow$ \\
\midrule
Vanilla VQGAN
& 82.7  & 2 & 0.1530 & 0.4051 & 87.82
& 8.0   & 2 & 0.1532 & 0.4083 & 81.01
& 2.2   & 2 & 0.1525 & 0.4064 & 83.59 \\
\qquad w/ $\beta=0$
& 3.7   & 248 & 0.2100 & 0.4959 & 161.42
& 0.3   & 253 & 0.3856 & 0.5903 & 300.06
& 0.0   & 254 & 0.4399 & 0.5911 & 368.17 \\
w/ Respawn
& 100.0 & 3 & 0.1484 & 0.3885 & 74.21
& 99.9  & 3 & 0.1401 & 0.3717 & 69.74
& 47.0  & 5 & 0.1395 & 0.3626 & 67.89 \\
\qquad w/ $\beta=0$
& 99.8  & 22 & 0.1566 & 0.3723 & 78.01
& 99.8  & 5 & 0.1362 & 0.3383 & 61.32
& 98.6  & 5 & 0.1300 & 0.3220 & 54.54 \\
$T_\text{wu} = $ 1k
& 100.0 & 5 & 0.1417 & 0.3553 & 69.11
& 100.0 & 4 & 0.1340 & 0.3380 & 61.86
& 98.4  & 5 & 0.1276 & 0.3134 & 52.55 \\
$T_\text{wu} = $ 5k
& 100.0 & 7 & 0.1410 & 0.3425 & 61.70
& 100.0 & 7 & 0.1286 & 0.3083 & 51.43
& 98.5  & 7 & \textbf{0.1241} & 0.2940 & 47.41 \\
$T_\text{wu} = $ 10k
& 100.0 & 7 & \textbf{0.1406} & 0.3391 & 63.35
& 99.9  & 10 & \textbf{0.1284} & 0.3016 & 51.97
& 98.0  & 12 & 0.1242 & 0.2841 & 47.83 \\
$T_\text{wu} = $ 15k
& 100.0 & 10 & 0.1416 & 0.3373 & 66.97
& 99.9  & 12 & 0.1288 & 0.2990 & 51.44
& 97.7  & 13 & 0.1244 & 0.2817 & 47.27 \\
$T_\text{wu} = $ 20k
& 100.0 & 11 & 0.1416 & 0.3344 & \textbf{60.91}
& 99.8  & 13 & 0.1291 & 0.2965 & 52.58
& 97.6  & 14 & 0.1248 & 0.2806 & 47.40 \\
$T_\text{wu} = $ 30k
& 100.0 & 14 & 0.1418 & \textbf{0.3323} & 61.75
& 99.8  & 15 & 0.1300 & 0.2946 & \textbf{49.84}
& 97.2  & 16 & 0.1255 & 0.2790 & 46.70 \\
$T_\text{wu} = $ 40k
& 100.0 & 19 & 0.1433 & 0.3341 & 61.26
& 99.8  & 17 & 0.1314 & \textbf{0.2939} & 50.59
& 97.0  & 18 & 0.1258 & \textbf{0.2787} & \textbf{43.89} \\
\bottomrule
\end{tabular}
\label{tab:vqgan_ablation_beta0}
\end{table}

%% file: references.bib
@book{polyanskiy2025information,
place={Cambridge}, title={Information Theory: From Coding to Learning}, publisher={Cambridge University Press}, author={Polyanskiy, Yury and Wu, Yihong}, year={2025}
}

@inproceedings{vandenoord2017vqvae,
 author = {van den Oord, Aaron and Vinyals, Oriol and Kavukcuoglu, Koray},
 booktitle = {Advances in Neural Information Processing Systems},
 editor = {I. Guyon and U. Von Luxburg and S. Bengio and H. Wallach and R. Fergus and S. Vishwanathan and R. Garnett},
 pages = {},
 publisher = {Curran Associates, Inc.},
 title = {Neural Discrete Representation Learning},
 url = {https://proceedings.neurips.cc/paper_files/paper/2017/file/7a98af17e63a0ac09ce2e96d03992fbc-Paper.pdf},
 volume = {30},
 year = {2017}
}

@inproceedings{razavi2019vqvae2,
 author = {Razavi, Ali and van den Oord, Aaron and Vinyals, Oriol},
 booktitle = {Advances in Neural Information Processing Systems},
 editor = {H. Wallach and H. Larochelle and A. Beygelzimer and F. d\textquotesingle Alch\'{e}-Buc and E. Fox and R. Garnett},
 pages = {},
 publisher = {Curran Associates, Inc.},
 title = {Generating Diverse High-Fidelity Images with VQ-VAE-2},
 url = {https://proceedings.neurips.cc/paper_files/paper/2019/file/5f8e2fa1718d1bbcadf1cd9c7a54fb8c-Paper.pdf},
 volume = {32},
 year = {2019}
}

@InProceedings{esser2021vqgan,
    author    = {Esser, Patrick and Rombach, Robin and Ommer, Bjorn},
    title     = {Taming Transformers for High-Resolution Image Synthesis},
    booktitle = {Proceedings of the IEEE/CVF Conference on Computer Vision and Pattern Recognition (CVPR)},
    month     = {June},
    year      = {2021},
    pages     = {12873-12883}
}

@InProceedings{takida2022sqvae,
  title = 	 {{SQ}-{VAE}: Variational {B}ayes on Discrete Representation with Self-annealed Stochastic Quantization},
  author =       {Takida, Yuhta and Shibuya, Takashi and Liao, Weihsiang and Lai, Chieh-Hsin and Ohmura, Junki and Uesaka, Toshimitsu and Murata, Naoki and Takahashi, Shusuke and Kumakura, Toshiyuki and Mitsufuji, Yuki},
  booktitle = 	 {Proceedings of the 39th International Conference on Machine Learning},
  pages = 	 {20987--21012},
  year = 	 {2022},
  editor = 	 {Chaudhuri, Kamalika and Jegelka, Stefanie and Song, Le and Szepesvari, Csaba and Niu, Gang and Sabato, Sivan},
  volume = 	 {162},
  series = 	 {Proceedings of Machine Learning Research},
  month = 	 {17--23 Jul},
  publisher =    {PMLR},
  url = 	 {https://proceedings.mlr.press/v162/takida22a.html}
}

@inproceedings{mentzer2024fsq,
     author = {Mentzer, Fabian and Minnen, David and Agustsson, Eirikur  and Tschannen, Michael},
     booktitle = {International Conference on Learning Representations},
     editor = {B. Kim and Y. Yue and S. Chaudhuri and K. Fragkiadaki and M. Khan and Y. Sun},
     pages = {51772--51783},
     title = {Finite Scalar Quantization: VQ-VAE Made Simple},
     url = {https://proceedings.iclr.cc/paper_files/paper/2024/file/e2dd53601de57c773343a7cdf09fae1c-Paper-Conference.pdf},
     volume = {2024},
     year = {2024}
}

@article{dhariwal2020jukebox,
  title={Jukebox: A Generative Model for Music},
  author={Prafulla Dhariwal and Heewoo Jun and Christine Payne and Jong Wook Kim and Alec Radford and Ilya Sutskever},
  journal={ArXiv},
  year={2020},
  volume={abs/2005.00341},
  url={https://api.semanticscholar.org/CorpusID:218470180}
}

@inproceedings{kumar2023dac,
 author = {Kumar, Rithesh and Seetharaman, Prem and Luebs, Alejandro and Kumar, Ishaan and Kumar, Kundan},
 booktitle = {Advances in Neural Information Processing Systems},
 doi = {10.52202/075280-1214},
 editor = {A. Oh and T. Naumann and A. Globerson and K. Saenko and M. Hardt and S. Levine},
 pages = {27980--27993},
 publisher = {Curran Associates, Inc.},
 title = {High-Fidelity Audio Compression with Improved RVQGAN},
 url = {https://proceedings.neurips.cc/paper_files/paper/2023/file/58d0e78cf042af5876e12661087bea12-Paper-Conference.pdf},
 volume = {36},
 year = {2023}
}

@inproceedings{zhang2023speechtokenizer,
title={SpeechTokenizer: Unified Speech Tokenizer for Speech Language Models},
author={Xin Zhang and Dong Zhang and Shimin Li and Yaqian Zhou and Xipeng Qiu},
booktitle={The Twelfth International Conference on Learning Representations},
year={2024},
url={https://openreview.net/forum?id=AF9Q8Vip84}
}

@article{defossez2024moshi,
  title={Moshi: a speech-text foundation model for real-time dialogue},
  author={Alexandre D{\'e}fossez and Laurent Mazar{\'e} and Manu Orsini and Am{\'e}lie Royer and Patrick P{\'e}rez and Herv{\'e} J{\'e}gou and Edouard Grave and Neil Zeghidour},
  journal={ArXiv},
  year={2024},
  volume={abs/2410.00037},
  url={https://api.semanticscholar.org/CorpusID:273022979}
}

@article{lancucki2020robust,
  title={Robust Training of Vector Quantized Bottleneck Models},
  author={Adrian Lancucki and Jan Chorowski and Guillaume Sanchez and Ricard Marxer and Nanxin Chen and Hans J. G. A. Dolfing and Sameer Khurana and Tanel Alum{\"a}e and Antoine Laurent},
  journal={2020 International Joint Conference on Neural Networks (IJCNN)},
  year={2020},
  pages={1-7},
  url={https://api.semanticscholar.org/CorpusID:218673823}
}

@article{zeghidour2021soundstream,
  title={SoundStream: An End-to-End Neural Audio Codec},
  author={Neil Zeghidour and Alejandro Luebs and Ahmed Omran and Jan Skoglund and Marco Tagliasacchi},
  journal={IEEE/ACM Transactions on Audio, Speech, and Language Processing},
  year={2021},
  volume={30},
  pages={495-507},
  url={https://api.semanticscholar.org/CorpusID:236149944}
}

@InProceedings{zheng2023cvqvae,
    author    = {Zheng, Chuanxia and Vedaldi, Andrea},
    title     = {Online Clustered Codebook},
    booktitle = {Proceedings of the IEEE/CVF International Conference on Computer Vision (ICCV)},
    month     = {October},
    year      = {2023},
    pages     = {22798-22807}
}

@inproceedings{zhu2024vqganlc,
 author = {Zhu, Lei and Wei, Fangyun and Lu, Yanye and Chen, Dong},
 booktitle = {Advances in Neural Information Processing Systems},
 doi = {10.52202/079017-0401},
 editor = {A. Globerson and L. Mackey and D. Belgrave and A. Fan and U. Paquet and J. Tomczak and C. Zhang},
 pages = {12612--12635},
 publisher = {Curran Associates, Inc.},
 title = {Scaling the Codebook Size of VQ-GAN to 100,000 with a Utilization Rate of 99\%},
 url = {https://proceedings.neurips.cc/paper_files/paper/2024/file/1716d022edeac750e57a2986a7135e13-Paper-Conference.pdf},
 volume = {37},
 year = {2024}
}

@misc{zhao2024representation,
      title={Representation Collapsing Problems in Vector Quantization}, 
      author={Wenhao Zhao and Qiran Zou and Rushi Shah and Dianbo Liu},
      year={2024},
      eprint={2411.16550},
      archivePrefix={arXiv},
      primaryClass={cs.LG},
      url={https://arxiv.org/abs/2411.16550}, 
}

@inproceedings{zhu2025simvq,
  author={Zhu, Yongxin and Li, Bocheng and Xin, Yifei and Xia, Zhihua and Xu, Linli},
  booktitle={2025 IEEE/CVF International Conference on Computer Vision (ICCV)}, 
  title={Addressing Representation Collapse in Vector Quantized Models with One Linear Layer}, 
  year={2025},
  volume={},
  number={},
  pages={22968-22977},
  doi={10.1109/ICCV51701.2025.02132}
}

@inproceedings{zhang2025dcvq,
title={Dimensional Collapse in {VQVAE}s: Evidence and Remedies},
author={Jiayou Zhang and Yifan Shen and Guangyi Chen and Le Song and Eric P. Xing},
booktitle={The Thirty-ninth Annual Conference on Neural Information Processing Systems},
year={2025},
url={https://openreview.net/forum?id=vJtnJfS2mQ}
}

@article{earlyquant2026,
  title={Early Quantization Shrinks Codebook: A Simple Fix for Diversity-Preserving Tokenization},
  author={Wenhao Zhao and Qiran Zou and Rushi Shah and Yudi Wu and Zhouhan Lin and Dianbo Liu},
  journal={ArXiv},
  year={2026},
  volume={abs/2603.17052},
  url={https://api.semanticscholar.org/CorpusID:286638547}
}

@misc{saxe2014exact,
      title={Exact solutions to the nonlinear dynamics of learning in deep linear neural networks}, 
      author={Andrew M. Saxe and James L. McClelland and Surya Ganguli},
      year={2014},
      eprint={1312.6120},
      archivePrefix={arXiv},
      primaryClass={cs.NE},
      url={https://arxiv.org/abs/1312.6120}, 
}

@inproceedings{refinetti2022dynamics,
  title = 	 {The dynamics of representation learning in shallow, non-linear autoencoders},
  author =       {Refinetti, Maria and Goldt, Sebastian},
  booktitle = 	 {Proceedings of the 39th International Conference on Machine Learning},
  pages = 	 {18499--18519},
  year = 	 {2022},
  editor = 	 {Chaudhuri, Kamalika and Jegelka, Stefanie and Song, Le and Szepesvari, Csaba and Niu, Gang and Sabato, Sivan},
  volume = 	 {162},
  series = 	 {Proceedings of Machine Learning Research},
  month = 	 {17--23 Jul},
  publisher =    {PMLR},
  url = 	 {https://proceedings.mlr.press/v162/refinetti22a.html}
}

@inproceedings{saeki2022utmos,
  title     = {{UTMOS}: {UTokyo}-{SaruLab} System for {VoiceMOS} Challenge 2022},
    author = {Saeki, Takaaki and Xin, Detai and Nakata, Wataru and Koriyama, Tomoki and Takamichi, Shinnosuke and Saruwatari, Hiroshi},
    year = {2022},
    month = {09},
    pages = {4521-4525},
    doi = {10.21437/Interspeech.2022-439}
}

@inproceedings{zhang2018unreasonable,
author = { Zhang, Richard and Isola, Phillip and Efros, Alexei A. and Shechtman, Eli and Wang, Oliver },
booktitle = { 2018 IEEE/CVF Conference on Computer Vision and Pattern Recognition (CVPR) },
title = {{ The Unreasonable Effectiveness of Deep Features as a Perceptual Metric }},
year = {2018},
volume = {},
ISSN = {},
pages = {586-595},
doi = {10.1109/CVPR.2018.00068},
url = {https://doi.ieeecomputersociety.org/10.1109/CVPR.2018.00068},
publisher = {IEEE Computer Society},
address = {Los Alamitos, CA, USA},
month =Jun
}

@inproceedings{rix2001pesq,
  author={Rix, A.W. and Beerends, J.G. and Hollier, M.P. and Hekstra, A.P.},
  booktitle={2001 IEEE International Conference on Acoustics, Speech, and Signal Processing. Proceedings (Cat. No.01CH37221)}, 
  title={Perceptual evaluation of speech quality (PESQ)-a new method for speech quality assessment of telephone networks and codecs}, 
  year={2001},
  volume={2},
  number={},
  pages={749-752 vol.2},
  doi={10.1109/ICASSP.2001.941023}
}

@inproceedings{taal2010stoi,
  author={Taal, Cees H. and Hendriks, Richard C. and Heusdens, Richard and Jensen, Jesper},
  booktitle={2010 IEEE International Conference on Acoustics, Speech and Signal Processing}, 
  title={A short-time objective intelligibility measure for time-frequency weighted noisy speech}, 
  year={2010},
  volume={},
  number={},
  pages={4214-4217},
  doi={10.1109/ICASSP.2010.5495701}}

@inproceedings{morrison2022chunked,
title={Chunked Autoregressive {GAN} for Conditional Waveform Synthesis},
author={Max Morrison and Rithesh Kumar and Kundan Kumar and Prem Seetharaman and Aaron Courville and Yoshua Bengio},
booktitle={International Conference on Learning Representations},
year={2022},
url={https://openreview.net/forum?id=v3aeIsY_vVX}
}

@inproceedings{ji2025wavtokenizer,
title={WavTokenizer: an Efficient Acoustic Discrete Codec Tokenizer for Audio Language Modeling},
author={Shengpeng Ji and Ziyue Jiang and Wen Wang and Yifu Chen and Minghui Fang and Jialong Zuo and Qian Yang and Xize Cheng and Zehan Wang and Ruiqi Li and Ziang Zhang and Xiaoda Yang and Rongjie Huang and Yidi Jiang and Qian Chen and Siqi Zheng and Zhou Zhao},
booktitle={The Thirteenth International Conference on Learning Representations},
year={2025},
url={https://openreview.net/forum?id=yBlVlS2Fd9}
}

@inproceedings{tian2020contrastive,
author = {Tian, Yonglong and Krishnan, Dilip and Isola, Phillip},
title = {Contrastive Multiview Coding},
year = {2020},
isbn = {978-3-030-58620-1},
publisher = {Springer-Verlag},
address = {Berlin, Heidelberg},
url = {https://doi.org/10.1007/978-3-030-58621-8_45},
doi = {10.1007/978-3-030-58621-8_45},
booktitle = {Computer Vision – ECCV 2020: 16th European Conference, Glasgow, UK, August 23–28, 2020, Proceedings, Part XI},
pages = {776–794},
numpages = {19},
location = {Glasgow, United Kingdom}
}

@article{wang2025tokenbridge,
  title={Bridging Continuous and Discrete Tokens for Autoregressive Visual Generation},
  author={Yuqing Wang and Zhijie Lin and Yao Teng and Yuanzhi Zhu and Shuhuai Ren and Jiashi Feng and Xihui Liu},
  journal={2025 IEEE/CVF International Conference on Computer Vision (ICCV)},
  year={2025},
  pages={18596-18605},
  url={https://api.semanticscholar.org/CorpusID:277151126}
}

@inproceedings{zen2019libritts,
  title={LibriTTS: A Corpus Derived from LibriSpeech for Text-to-Speech},
  author={Heiga Zen and Viet Dang and Robert A. J. Clark and Yu Zhang and Ron J. Weiss and Ye Jia and Z. Chen and Yonghui Wu},
  booktitle={Interspeech},
  year={2019},
  url={https://api.semanticscholar.org/CorpusID:102352475}
}
